%% file: main.tex
\documentclass[twoside]{article}

\usepackage[accepted]{aistats2026}
\usepackage[round]{natbib}

\usepackage{hyperref}
\usepackage{amsmath,amssymb,amsthm}
\usepackage[ruled,vlined]{algorithm2e}
\usepackage{graphicx}
\usepackage{booktabs}
\usepackage{subfigure}
\usepackage{subcaption}
\usepackage{multirow}

\DeclareMathOperator*{\argmax}{argmax}
\DeclareMathOperator*{\argmin}{argmin}

\begin{document}

%

%

\twocolumn[

\aistatstitle{High-dimensional Level Set Estimation with Trust Regions and Double Acquisition Functions}

\aistatsauthor{ Giang Ngo \And Dat Phan Trong \And Dang Nguyen \And Sunil Gupta }

\aistatsaddress{ Applied Artificial Intelligence Initiative, Deakin University } ]

\begin{abstract}
Level set estimation (LSE) classifies whether an unknown function's value exceeds a specified threshold for given inputs, a fundamental problem in many real-world applications.
In active learning settings with limited initial data, we aim to iteratively acquire informative points to construct an accurate classifier for this task.
In high-dimensional spaces, this becomes challenging where the search volume grows exponentially with increasing dimensionality.
We propose TRLSE, an algorithm for high-dimensional LSE, which identifies and refines regions near the threshold boundary with dual acquisition functions operating at both global and local levels.
We provide a theoretical analysis of TRLSE's accuracy and show its superior sample efficiency against existing methods through extensive evaluations on multiple synthetic and real-world LSE problems.
\end{abstract}

\input{1_intro.tex}
\input{2_related_works.tex}
\input{3_method.tex}
\input{4_theoretical_analysis.tex}
\input{5_experiments.tex}
\input{6_conclusion.tex}

\bibliography{main}
\bibliographystyle{plainnat}

\section*{Checklist}

\begin{enumerate}

  \item For all models and algorithms presented, check if you include:
  \begin{enumerate}
    \item A clear description of the mathematical setting, assumptions, algorithm, and/or model. [\textbf{Yes}/No/Not Applicable]
    \item An analysis of the properties and complexity (time, space, sample size) of any algorithm. [\textbf{Yes}/No/Not Applicable]
    \item (Optional) Anonymized source code, with specification of all dependencies, including external libraries. [\textbf{Yes}/No/Not Applicable]
  \end{enumerate}

  \item For any theoretical claim, check if you include:
  \begin{enumerate}
    \item Statements of the full set of assumptions of all theoretical results. [\textbf{Yes}/No/Not Applicable]
    \item Complete proofs of all theoretical results. [\textbf{Yes}/No/Not Applicable]
    \item Clear explanations of any assumptions. [\textbf{Yes}/No/Not Applicable]     
  \end{enumerate}

  \item For all figures and tables that present empirical results, check if you include:
  \begin{enumerate}
    \item The code, data, and instructions needed to reproduce the main experimental results (either in the supplemental material or as a URL). [\textbf{Yes}/No/Not Applicable]
    \item All the training details (e.g., data splits, hyperparameters, how they were chosen). [\textbf{Yes}/No/Not Applicable]
    \item A clear definition of the specific measure or statistics and error bars (e.g., with respect to the random seed after running experiments multiple times). [\textbf{Yes}/No/Not Applicable]
    \item A description of the computing infrastructure used. (e.g., type of GPUs, internal cluster, or cloud provider). [\textbf{Yes}/No/Not Applicable]
  \end{enumerate}

  \item If you are using existing assets (e.g., code, data, models) or curating/releasing new assets, check if you include:
  \begin{enumerate}
    \item Citations of the creator If your work uses existing assets. [\textbf{Yes}/No/Not Applicable]
    \item The license information of the assets, if applicable. [Yes/No/\textbf{Not Applicable}]
    \item New assets either in the supplemental material or as a URL, if applicable. [Yes/No/\textbf{Not Applicable}]
    \item Information about consent from data providers/curators. [Yes/No/\textbf{Not Applicable}]
    \item Discussion of sensible content if applicable, e.g., personally identifiable information or offensive content. [Yes/No/\textbf{Not Applicable}]
  \end{enumerate}

  \item If you used crowdsourcing or conducted research with human subjects, check if you include:
  \begin{enumerate}
    \item The full text of instructions given to participants and screenshots. [Yes/No/\textbf{Not Applicable}]
    \item Descriptions of potential participant risks, with links to Institutional Review Board (IRB) approvals if applicable. [Yes/No/\textbf{Not Applicable}]
    \item The estimated hourly wage paid to participants and the total amount spent on participant compensation. [Yes/No/\textbf{Not Applicable}]
  \end{enumerate}

\end{enumerate}

\clearpage
\appendix
\thispagestyle{empty}

\onecolumn
\aistatstitle{Appendices}

\input{7_appendix_proof.tex}
\input{7_appendix_extra_exp.tex}

\end{document}

%% file: 1_intro.tex
\section{Introduction}
Given a threshold \(h\) and a black-box function \(f\), the \textit{superlevel} (or \textit{sublevel}) set of \(f\) comprises all points \(\mathbf{x}\) where \(f(\mathbf{x})\geq h\) (or \(f(\mathbf{x})<h\)). 
Building a classifier to classify points into either set is challenging when \(f\) is an expensive-to-evaluate black-box function leading to limited training data. 
Examples of such applications include computational chemistry \citep{10.1063/5.0044989}, environmental monitoring \citep{activelse}, and vehicle structure design \citep{c2lse}, where data collection is costly.
An active learning solution to LSE needs to iteratively evaluate informative points to improve the classifier.
Thus, sample efficiency is a key objective when designing LSE algorithms, alongside achieving high classification accuracy.

LSE algorithms select the next point by optimizing an \textit{acquisition function} (AF), which quantifies utility for the classifier using a surrogate model (typically a Gaussian Process (GP)) fitted on existing data points \citep{rmile}.
In high-dimensional LSE (HDLSE), this scheme is challenged by GP's performance degradation in high-dimensional spaces, but a recent work by \citet{hvarfner2024vanillabayesianoptimizationperforms} shows that adjusting GP kernel lengthscales can significantly alleviate this issue.
However, challenges in HDLSE go beyond GP modeling.
First, optimizing the AF, often a non-convex problem, becomes harder as the search space expands exponentially. 
Second, GP inference scales cubically with data size, making AF computation slow as HDLSE accumulates more points.

Bayesian Optimization (BO), a related problem which optimizes \(f\) instead, faces similar challenges in a high-dimensional setting.
Local BO methods \citep{NEURIPS2019_6c990b7a, daulton2022multi} have become a popular solution in this setting.
They optimize the AF over local trust regions (TRs) instead of the entire global search space, and each TR is accurately modeled by its own local GP. 
Multiple TRs are first randomly initialized, and their expansion/shrinkage is dictated by whether they can improve upon the incumbent solution.
A TR without improvement will be replaced by a randomly initialized one to ensure eventual global optimization.
An advantage of local BO is that optimizing the AF over a bounded local region is easier than over the entire space, with less over-exploration and faster computation.

However, adapting the local BO framework to HDLSE is not straightforward due to fundamental differences between the two problems.
First, defining the success of a trust region (TR) for LSE is non-trivial. 
TR performance in BO is tied to improvement over the incumbent solution, but measuring such TR performances in LSE is unclear, complicating the ability to expand, shrink, or replace underperforming TRs effectively.
Second, it is also ambiguous where a TR's center should be located. 
In BO, TRs are naturally centered at the incumbent solutions or the optima of a confidence bound.
No such concepts exist in LSE, leaving the placement of TRs' centers uncertain.
Finally, local BO only focuses on promising regions to find the next best solution inside those regions, while LSE demands accurate classification across the entire space. 
This creates a significant gap: while points within TRs might be well-modeled and classified, those outside the TRs, despite being the majority since TRs are often small, are neglected. 

We address the challenges above of adapting local BO for HDLSE tasks with a novel method, referred to as \textbf{Trust-region Level Set Estimation} (TRLSE). 
To reveal the location of the level sets, it is important to identify the \textit{threshold boundary} region where \(f(\textbf{x})\approx h\) helps reveal the location of the level sets.
Therefore, we focus the TRs at this boundary and expand a region if it is located around the threshold boundary. 
TRLSE employs two levels of acquisition functions: a \textit{global} AF to identify regions likely to contain the threshold boundary, and a \textit{local} AF within each TR to capture boundary details with finer resolution. 
The global AF helps revealing the threshold boundary over the entire space, extending accurate classification to outside the local TRs.
Optimizing a local AF over bounded local regions enjoys the benefits mentioned above, including less over-exploration, faster computation, and a better model for the threshold boundary, which is often the most uncertain region to classify in an LSE problem.

Our contributions are summarized as follows:
\begin{enumerate}
    \item We propose TRLSE, a novel algorithm for HDLSE with multiple trust regions to reveal the threshold boundary and two levels of acquisition functions to guide region initialization and local sampling.
    \item We theoretically analyze TRLSE, guaranteeing its classification accuracy outside the trust regions, which accounts for most of the search space.
    \item We demonstrate TRLSE's superior performance on a variety of synthetic and real-world LSE problems across both low and high-dimensional search spaces (from 2 - 1000 dimensions).
\end{enumerate}

%% file: 2_related_works.tex
\section{Related Works}\label{related_works}
\paragraph{Level Set Estimation}
Most LSE methods focus on designing AFs that reveal information about the level sets of \(f\). 
The Straddle heuristic \citep{straddle}, LSE \citep{activelse} (later referred to as ActiveLSE to avoid confusion with the setting's abbreviation), and C2LSE \citep{c2lse} utilize confidence intervals to prioritize evaluating where the classification by the surrogate model is the most uncertain. 
These points are often near high-uncertainty regions and/or potential threshold boundaries.
One-step lookahead mechanisms are also used; for example, TruVar \citep{truvar} considers the potential reduction in truncated variance while unifying LSE and BO, and RMILE \citep{rmile} prioritizes points expected to expand the superlevel set volume.
While the idea sounds promising, AFs using one-step lookahead mechanism often suffer from high computational complexity.
Other approaches include an information-theoretic acquisition \citep{nguyen2021information} and sampling strategy based on experimental design \citep{mason2021nearly}.
Despite these alternatives, the Straddle heuristic is popularly employed in many real-world applications due to strong empirical performance and adaptability to both discrete and continuous LSE problems.
Nevertheless, optimizing AFs becomes increasingly challenging in high-dimensional spaces due to the exponential growth of the search space and the computational demands of GP surrogate models.
\paragraph{High-dimensional LSE}
HLSE \citep{ha2021high} used BNN with Monte Carlo dropout as the surrogate model instead of a GP. 
While a BNN generally scales better to high-dimensional inputs than a GP, HLSE suffers from remarkably higher runtime due to an extensive, ongoing hyperparameter tuning and lacks theoretical guarantees for solution accuracy. 
Though not specifically targeting HDLSE, \citet{multiscale} proposed a hierarchical space-partitioning method to explore the search space at different levels of detail with potentially smaller computational costs in higher dimensionalities. 
However, in truly high-dimensional spaces, the exponential growth in number of partitioned regions limits its effectiveness.
The method also offers only theoretical convergence guarantees, lacking both rigorous empirical evaluation and practical guidance for hyperparameter selection.
\paragraph{High-dimensional BO}
Many high-dimensional BO methods rely on certain assumptions such as the existence of a representative low-dimensional subspace \citep{rembo, pmlr-v97-kirschner19a}, an additive structure \citep{Han_Arora_Scarlett_2021, pmlr-v84-rolland18a}, or proximity of the optimum to the search space center \citep{pmlr-v80-oh18a}.
Without such assumptions, adjusting kernel lengthscales \citep{pmlr-v70-rana17a, hvarfner2024vanillabayesianoptimizationperforms} is also a common strategy which assumes a reduced complexity over \(f\).
Local BO \citep{zhao2022multiobjective, daulton2022multi} has become prominent recently with its scalability to large evaluation budgets while retaining the advantages of GPs as surrogate models, with TuRBO \citep{NEURIPS2019_6c990b7a} being one of the most widely adopted algorithms.
Using multiple local trust regions avoids over-exploration due to smaller bounded search spaces which allows faster, more accurate acquisition optimization.
However, local BO methods, designed for optimization, do not directly apply to LSE, which requires accurate classification of the entire input space rather than finding a single optimum.

%% file: 3_method.tex
\section{Methodology}
\subsection{Problem Definition}
Let \(f:\mathcal{X}\longrightarrow \mathbb{R}\) be a black-box function where \(\mathcal{X}\subset \mathbb{R}^d\) is a compact set. 
An LSE problem aims to classify any \(\mathbf{x}\in\mathcal{X}\) into either the superlevel set \(\mathcal{H}=\{\mathbf{x}\in \mathcal{X}: f(\mathbf{x})\geq h\}\) and the sublevel set \(\mathcal{L}=\{\mathbf{x}\in \mathcal{X}: f(\mathbf{x})<h\}\) for a given threshold \(h\).
More information about \(f\) can be obtained through noisy function evaluations \(y=f(\mathbf{x})+\eta\) at \(\mathbf{x}\) with observation noise \(\eta\sim\mathcal{N}(0,\sigma^2)\).
\subsection{Trust-region Level Set Estimation}
We now present TRLSE, a method utilizing multiple local trust regions for high-dimensional LSE. 
The method steps are summarized in Algorithm \ref{alg:one}, with a detailed description for each component presented in the following sections.
TRLSE simultaneously maintains \(R\) trust regions at every time step, and in the beginning, they are initialized by creating \(R\) hypercubes of volume \(V_{init}\) centered at \(R\) data points sampled uniformly within the search space \(\mathcal{X}\). 
At each iteration, the TRs are updated to center at the points deemed closest to the threshold boundary and to change their volume according to how well their function values range contains the threshold (see Section \ref{update}). 
Further, if a region is too "small", it will be replaced by a new region initialized by optimizing the global AF (see Section \ref{double_acq}). 
The algorithm then proceeds to refine the local regions by choosing the next evaluation point through jointly optimizing the local AF over all TRs (see Section \ref{double_acq}). 
Both AFs share the goal of balancing between exploring the search space and identifying the threshold boundary. 
The local AF allows a granular surrogate model of the threshold boundary, while the global one ensures coverage on the entire search space. 
It should be noted that the classification rules in the end of Algorithm \ref{alg:one} can be applied at any time step of the process, not necessarily at the end. 
These rules use the global GP to classify outside of the TRs and the local GPs to classify inside the TRs.
If a point belongs to more than one TRs, it will be classified by the local GP that gives the lowest posterior variance.

\begin{algorithm}[ht]
\caption{Trust-region level set estimation}
\label{alg:one}
\KwIn{threshold \(h\), budget \(B\), \(V_{init}\), \(V_{max}\), \#trust regions \(R\), confidence parameter \(\beta\).}
Initialize \(R\) regions \([\mathcal{T}^i]_{i=1}^R\) of volume \(V_{init}\) centered at \(D=[\mathcal{C}_0^i]_{i=1}^R\) with \(\mathcal{C}_0^i\sim \textup{Uniform}(\mathcal{X})\)\;
\For{$i=1$ \KwTo $R$}{
    Fit \(\mathcal{GP}_0^i\) for \(\mathcal{T}^i\) with \([\mathcal{C}_0^i,y_i]\), \(y_i=f(\mathcal{C}_0^i)+\eta\)\;
}
\(t\leftarrow 1; n\leftarrow R\); Fit \(\mathcal{GP}_g\) with \(D\)\;
\While{$|D| \leq B$}{
    \For{$i=1$ \KwTo $R$}{
        Update centroid \(\mathcal{C}_t^i\), volume \(V^i_t\), and length \(L^i_t\) of \(\mathcal{T}^i\) \tcp*{See Section \ref{update}}
        Update \(\mathcal{GP}_t^i\) with \(\{\mathbf{x}\in D\cap[\mathcal{C}_t^i\pm L_t^i]\}\)\;
    }
    \For{$i=1$ \KwTo $R$}{
        \If{$V^i_t<V_{init}/2$}{
            Discard \(\mathcal{T}^i\)\;
            Select the new centroid $\mathcal{C}_t^i$ at $\bar{\mathbf{x}}_n=\argmax_{\mathbf{x}\in \Bar{\mathcal{U}}_t} a_g(\mathbf{x})$ \tcp*{See Eq.\ref{acq_global}}
            Evaluate \(f\) at \(\bar{\mathbf{x}}_n:\bar{y}_n=f(\bar{\mathbf{x}}_n)+\eta\)\;
            Re-initialize \(\mathcal{T}_i\) at \(\mathcal{C}_t^i\) with volume \(V_{init}\)\;
            \(D=D\cup(\bar{\mathbf{x}}_n,\bar{y}_n)\), update \(\mathcal{GP}_g\) with \(D\)\; 
            \(n\leftarrow n+1\)\;
        }
    }
    Select \(\mathbf{x}_t=\argmax_{\mathbf{x}\in\mathcal{U}^t} a_l(\mathbf{x})\) \tcp*{See Eq.\ref{acq_local}}
    Get evaluation \(y_t\) of \(f\) at \(\mathbf{x}_t\); \(D=D\cup(\mathbf{x}_t, y_t)\)\;
    Update \(\mathcal{GP}_g\) with \(D\); \(t\leftarrow t+1\)\;
}
\KwOut{Classification rules: \\
\(\mathbf{x}\in\hat{\mathcal{H}}_T\) (predicted superlevel set) if \(\mu_g(\mathbf{x})\geq h\) if $\mathbf{x}\in \Bar{\mathcal{U}}_T$ or \(\mu_T^i(\mathbf{x})\geq h\) if $\mathbf{x}\in\mathcal{T}^i$\;
\(\mathbf{x}\in\hat{\mathcal{L}}_T\) (predicted sublevel set) if \(\mu_g(\mathbf{x})< h\) if $\mathbf{x}\in \Bar{\mathcal{U}}_T$ or \(\mu_T^i(\mathbf{x})< h\) if $\mathbf{x}\in\mathcal{T}^i$\;
}
\end{algorithm}
Below is a notation list for the reader's convenience:
\begin{itemize}
    \item \(\mathcal{T}^i\): the \(i\)-th TR.
    \item \(\mathcal{C}_t^i,V_t^i,L_t^i\): the centroid, volume, and lengths of \(\mathcal{T}^i\) at the end of iteration \(t\) respectively.
    \item \(\mathcal{GP}_t^i\): the local GP of \(\mathcal{T}^i\) at iteration \(t\) (with posterior mean \(\mu_t^i(\mathbf{x})\) and variance \((\sigma_t^i(\mathbf{x}))^2\)).
    \item \(\mathcal{U}_t=\bigcup\limits_{i=1}^{R}\mathcal{T}^i\) and \(\Bar{\mathcal{U}}_t=\mathcal{X}\setminus\mathcal{U}_t\).
    \item \(D\): the set of all evaluated data points.
    \item \(n\): total number of TRs ever initialized.
    \item \(\mathcal{GP}_g\): the global GP fitted with \(D\) (with posterior mean \(\mu_g(\mathbf{x})\) and variance \(\sigma_g^2(\mathbf{x})\) at \(\mathbf{x}\)).
\end{itemize}
\subsubsection{Trust Region Update}\label{update}
\paragraph{Moving Region Centroid}
To solve the centering problem for TRs, TRLSE updates each TR to cover the threshold boundary as much as possible given the available information about \(f\). 
Thus, the first step is to move the centroid of \(\mathcal{T}^i\) to the point deemed closest to the threshold boundary by \(\mathcal{GP}^i_{t-1}\) as follows:
\begin{align}
    \mathcal{C}_t^i=\argmin_{\mathbf{x}\in\mathcal{T}^i}|\mu_{t-1}^i(\mathbf{x})-h|.\label{move_centroid}
\end{align}
\paragraph{Updating Volume and Length}
Next, we define the update rule for the volume of \(\mathcal{T}^i\).
TuRBO simply doubles the TR's volume when it experiences a series of consecutive progress upon the incumbent solution (and halves the TR's volume with a series of failures).
To address the challenge of defining the success of a TR for LSE, we reward a TR if it centers at the threshold boundary (and shrinks it if it deviates from this boundary).
To determine if \(\mathcal{T}^i\) successfully brackets the threshold $h$, we define a penalty based on whether the confidence interval of the function's range within the TR contains \(h\) or not.
Specifically, we first define the penalty for a region \(\mathcal{T}^i\) at iteration \(t\) as:
\begin{align}
    \mathcal{P}_t(\mathcal{T}^i)=\Phi\left(\left|\frac{\Bar{l}_{t-1}^i+\Bar{u}_{t-1}^i-2h}{2\Bar{\sigma}_{t-1}^i}\right|\right)
\end{align}
with \(\Phi\) being the standard normal CDF and:
\begin{align*}
    \Bar{l}_{t-1}^i&=\min_{\mathbf{x}\in\mathcal{T}^i}\left(\mu_{t-1}^i(\mathbf{x})-\beta\sigma_{t-1}^i(\mathbf{x})\right)\\
    \Bar{u}_{t-1}^i&=\max_{\mathbf{x}\in\mathcal{T}^i}\left(\mu_{t-1}^i(\mathbf{x})+\beta\sigma_{t-1}^i(\mathbf{x})\right)\\
    \Bar{\sigma}_{t-1}^i&=\frac{\Bar{u}_{t-1}^i-\Bar{l}_{t-1}^i}{2\beta}
\end{align*}
where \(\beta\) is the confidence parameter. 
Using both \(\Bar{l}_{t-1}^i\) and \(\Bar{u}_{t-1}^i\) allows the update to take uncertainty (within \(\mathcal{T}^i\)) into account while giving a rigorous estimate of the function value's range. 
We use \(\Phi\) to limit the penalty to at most 1 to avoid a sudden, drastic change in region volume. 
Though this limit can slow down the discard of underperforming regions, it also protects newly created regions from being discarded prematurely, leaving a chance for local exploration.
\(\mathcal{P}_t(\mathcal{T}^i)\) increases when function values within \(\mathcal{T}^i\) are far from \(h\) and decreases when they are close to \(h\).

Instead of simply doubling/halving region volumes, we adjust them according to their penalty as follows:
\begin{align}
    V_t^i=\min(V_{t-1}^i S(\mathcal{P}_{t-1}(\mathcal{T}^i)), V_{max})\label{update_volume}
\end{align}
where \(S(\cdot)\) smoothly decreases from around 2 to 0 as the penalty \(\mathcal{P}_t\) increases from 0.5 to 1 (the feasible value range of \(\mathcal{P}_t\)) so that \(S(\mathcal{P}_t)\) serves as an adjustment factor for updating region volume. 
We discuss different choices for \(S(\cdot)\) in Section \ref{S_choices}.
Each TR is allowed to expand up to \(V_{max}\) to ensure it is sufficiently small for accurate acquisition optimization.
Similar to TuRBO, the region lengths are updated as:
\begin{align}
    L_t^i=(V_t^i/\prod_{j}^{d}\lambda_j)^{1/d}\cdot(\lambda_1,\lambda_2,...,\lambda_d)^T\label{update_lengthscale}
\end{align}
where \(\lambda_k\) is the lengthscale for the \(k\)-th dimension of \(\mathcal{GP}_{t-1}^i\). 
This allows the region to expand more along dimensions where the function is not so sensitive.
\paragraph{Updating Local GP}
Finally, \(\mathcal{GP}_t^i\) is fitted with any existing training data point \(\mathbf{x}\in D\) s.t. \(\mathbf{x}\in[\mathcal{C}_t^i - L_t^i,\mathcal{C}_t^i + L_t^i]\) (i.e. the hyperrectangle of length \(2L_t^i\) centered at \(\mathcal{C}_t^i\)). 
By using data points from immediately outside the TR, \(\mathcal{GP}_t^i\) models the TR's border better, which can be inaccurate if only data points inside the TR are used.
\subsubsection{Two Levels of Acquisition Functions}\label{double_acq}
Finally, TRLSE employs a dual AF system to ensure global accuracy. 
The local AF refines the boundary within TRs, while the global AF helps re-initialize failing TRs at informative locations, preventing the algorithm from neglecting any part of the search space.
\paragraph{Global AF for TR suggestion}
A TR is discarded when its volume falls below \(V_{init}/2\), suggesting it is unlikely to contain the threshold boundary. 
The data points collected within this region are retained in the training set. 
Given \(n-1\) regions initialized so far, the centroid \(\mathcal{C}_t^i\) of the new TR will be at:
\begin{align}
    \bar{\mathbf{x}}_n=\underset{\mathbf{x}\in\Bar{\mathcal{U}}_t}{\textup{argmax } }a_g(\mathbf{x})\label{acq_global}
\end{align}
with \(a_g\) being an AF (calculated using \(\mathcal{GP}_g\)) that increases its value at locations that are either associated with high uncertainty (i.e. exploration) or deemed close to the threshold boundary (i.e. exploitation). 
We choose the Straddle heuristic \citep{straddle}, which can be easily optimized over continuous spaces with a low computational cost:
\begin{align}
    a_g(\mathbf{x})=\beta\sigma_g(\mathbf{x})-|\mu_g(\mathbf{x})-h|.\label{eqn:straddle}
\end{align}
The function value \(\bar{y}_n\) at \(\bar{\mathbf{x}}_n\) is evaluated, and \(\mathcal{GP}_t^i\) is re-fitted as in the region update step.
This replacement ensures no underperforming region remains for too long and becomes wasteful in the local acquisition optimization step, and that eventually all TRs remaining will focus on the threshold boundary.
\paragraph{Local AF for refining regions}
Local modeling is further enriched by evaluating points from within \(\mathcal{U}_t\):
\begin{align}
    \mathbf{x}_t&=\argmax_{x_l^i:1\leq i \leq R} a_l^i\label{acq_local}\\
    \text{ where }a_l^i&=\max_{\mathbf{x}\in\mathcal{T}_i} a_l(\mathbf{x})\text{ and }x_l^i=\argmax_{\mathbf{x}\in\mathcal{T}_i} a_l(\mathbf{x})\nonumber
\end{align}
where \(a_l(\mathbf{x})\) is defined as \(a_g(\mathbf{x})\) but uses the local posterior mean \(\mu_t^i(\mathbf{x})\) and standard deviation \(\sigma_t^i(\mathbf{x})\).
Sampling inside the TRs allows a better model for the detected threshold boundary.

%% file: 4_theoretical_analysis.tex
\section{Theoretical Analysis}
\theoremstyle{plain}
\newtheorem{theorem}{Theorem}[section]
\newtheorem{proposition}[theorem]{Proposition}
\newtheorem{lemma}[theorem]{Lemma}
\newtheorem{corollary}[theorem]{Corollary}
\theoremstyle{definition}
\newtheorem{definition}[theorem]{Definition}
\newtheorem{assumption}[theorem]{Assumption}
\theoremstyle{remark}
\newtheorem{remark}[theorem]{Remark}
We now theoretically analyze TRLSE's accuracy.
All proofs can be found in Appendix \ref{proof}.
\subsection{Ensuring Progress via Trust Region Re-initialization}\label{sec:trust_region_reinitialization}
TRLSE initializes regions for global coverage and samples within them for local accuracy, so its performance depends on both the number of regions initialized and samples collected.
While the latter inherently increases over the iterations, it is not straightforward that the former will increase (i.e. new regions may not be initialized).
We first show that underperforming TRs are surely discarded, which means initialization of new regions will happen. 
\begin{lemma}\label{first_lemma}
    Let \(\beta\geq\Phi^{-1}(\psi)\) with \(\psi\in(0.5,1)\) and \(S(u)=2/(1+\textup{exp}(au-b))\) with \(b=\psi a\).
    Any region \(\mathcal{T}^i\) where either \(\Bar{l}_{t-1}^i\geq h\) or \(\Bar{u}_{t-1}^i\leq h\) and remains so for the next iterations will be replaced after at most \(\zeta\) iterations with \(\zeta=\log(V_{init}/V_{max}^2)/(\log(2)-\log(1+\exp(-b+a\Phi(\beta))))\).
\end{lemma}
Since new TRs will be initialized, \(n\) does not get stuck.
Suppose that when TRLSE terminates at iteration \(T\), it has initialized \(N\) TRs.
The accuracy of TRLSE can now be estimated by \(N\) and \(T\).
\subsection{Guarantee for Classification Accuracy}
\subsubsection{Accuracy Definition}
We restate the definition for the accuracy of level set classification as follows:
\begin{definition}
    Let \(\hat{\mathcal{H}}\) and \(\hat{\mathcal{L}}\) be the predicted superlevel and sublevel sets of an LSE algorithm \(\mathbb{A}\) for a black-box function \(f\) over the search space \(\mathcal{X}\) such that they are disjoint and that \(\hat{\mathcal{H}}\cup\hat{\mathcal{L}}=\mathcal{X}\). 
    Given a non-negative constant \(\epsilon\), the level set classification of a point \(\mathbf{x}\in\mathcal{X}\) is said to be \(\epsilon\)-accurate (or \(\mathbf{x}\) is classified with \(\epsilon\)-accuracy) if \(h-f(\mathbf{x})\leq\epsilon\;\forall \mathbf{x}\in\hat{\mathcal{H}}\) and \(f(\mathbf{x})-h\leq\epsilon\;\forall \mathbf{x}\in\hat{\mathcal{L}}\).
\end{definition}
The constant \(\epsilon\) is often predefined per application to indicate an acceptable error margin (e.g., if \(h=1.0\), an \(\epsilon\) of 0.01 means that the practitioner accepts classification errors for points whose function values are within \([0.99,1.01]\)). For example, in safety-critical applications, \(\epsilon\) may be set to a very small value to ensure high accuracy, while in less critical applications, a larger value may be acceptable to reduce the evaluation budget.

A point-wise definition of accuracy is a necessary adaptation for the continuous setting, moving beyond the definition by \citet{activelse} for predefined finite sets.
It serves as a fundamental building block from which one can aggregate performance over a finite test set drawn from anywhere inside $\mathcal{X}$.
Furthermore, a guarantee on the continuous space (e.g., "$\mathcal{X}$ is classified with $\epsilon$-accuracy if the error for every point is less than $\epsilon$"), while seemingly stronger, does not provide a clear mechanism for performance analysis on a finite test set, leading to a lack of practical utility.
\subsubsection{Accuracy Guarantee for TRLSE}
To analyze the accuracy convergence of TRLSE, we follow \citet{srinivas} and assume that the unknown black-box function $f: \mathcal{X} \to \mathbb{R}$ is a sample path drawn from a Gaussian Process (GP) prior, denoted as $\mathcal{GP}(\mu_0(\mathbf{x}), k(\mathbf{x}, \mathbf{x}'))$, where $\mu_0(\mathbf{x})$ is the mean function and $k(\mathbf{x}, \mathbf{x}')$ is a positive definite covariance (kernel) function. 
Without loss of generality, we assume the mean function is zero. 

The following result shows the theoretical convergence for the level set accuracy of TRLSE:
\begin{theorem}[Accuracy guarantee of TRLSE]
\label{main_theorem}
Let TRLSE be run with the Straddle heuristic. 
Assume that there exist constants $\kappa_g \in (0, 1]$ such that $a_g(\bar{\mathbf{x}}_n) \ge \kappa_g \max_{\mathbf{x} \in \bar{\mathcal{U}}_t} a_g(\mathbf{x})$ for all $n\geq 1$.
For any $N$ that satisfies the condition $a_g(\bar{\mathbf{x}}_N) \le \frac{1}{N}\sum_{n=1}^{N} a_g(\bar{\mathbf{x}}_n)$, it holds with a probability of at least $\Phi(\beta)$ that the classification of a point $\mathbf{x}$ outside the trust regions ($\mathbf{x} \in \bar{\mathcal{U}}_T$) is $\epsilon$-accurate, where $\epsilon$ satisfies:
\begin{align}
    \epsilon \leq \frac{\beta}{\kappa_g} \sqrt{\frac{C_1 \gamma_N}{N}},\textup{ with }C_1=\frac{2}{\log(1+\sigma^{-2})}.\label{global_accuracy}
\end{align}
\end{theorem}
Assuming the use of common kernels as in Theorem 5 in \citet{srinivas} (i.e. linear, squared exponential, and Matérn kernels), the right hand side of Eq.\ref{global_accuracy} converges to zero as \(N\) increases, which means the accuracy of TRLSE outside the TRs improves over time.
Bounds for the maximum information gain \(\gamma_N\) w.r.t. the number of dimensions \(d\) and the amount of data points are given by \citet{srinivas}.

Theorem \ref{main_theorem} guarantees the accuracy \textit{outside} the TRs (i.e. $\forall\mathbf{x} \in \bar{\mathcal{U}}_T$), not inside them.
In high-dimensional problems, the total volume occupied by the TRs maintained by TRLSE is often insignificant, as can be seen in Appendix \ref{detail_experiment}. 
Deriving a formal accuracy guarantee for points within the trust regions (i.e. for all $\mathbf{x} \in \mathcal{U}_T$) is a significant challenge, as their position, size, and local models change at every iteration.
Given this complexity and the limited practical interest in the accuracy over such insignificant volume inside the TRs, the focus remains on providing a foundational accuracy guarantee for the vast global space (as delivered by Theorem \ref{main_theorem}), which is the most critical measure of the algorithm's performance.
\paragraph{Proof sketch} 
We first need to establish that the acquisition values $a_g(\bar{\mathbf{x}}_n)$ and $a_t(\mathbf{x}_t)$ are always non-negative under a mild assumption that all data points are not yet confidently classified by the algorithm.
With $A_N=\frac{1}{N}\sum_{n=1}^{N} a_g(\bar{\mathbf{x}}_n)$, we can link the sum of the global acquisition values to the sum of the global posterior variances and eventually maximum information gain over sampling any $N$ points using the submodularity property of information gain.
This allows us to upper bound $A_N$ by a constant proportional to $\sqrt{\gamma_N/N}$.
After revealing the connection between the acquisition value and the classification accuracy on the same search space, we can now derive the accuracy outside the trust regions.

\subsubsection{Practical Considerations}
The following lemma formalizes that the condition $a_g(\bar{\mathbf{x}}_N) \le A_N$ in Theorem \ref{main_theorem} occurs infinitely often.
\begin{lemma}\label{lem:l6_infinitely_many_non_aggressive}
There are infinitely many integers $N$ s.t. $a_g(\bar{\mathbf{x}}_N) \le A_N$.
\end{lemma}
Lemma \ref{lem:l6_infinitely_many_non_aggressive} means this condition is not a restrictive or rare event, which is crucial for the convergence of the global accuracy.
In practice, with the magnitude of $\sigma_g(\mathbf{x})$ decreasing with more data points, the global acquisition value $a_g(\bar{\mathbf{x}}_N)$ will also decrease, which means that it will be less than or equal to its running average.
Therefore, while the theorem is stated conditionally, the convergence of the global accuracy is primarily dictated by the rate shown in Eq.\ref{global_accuracy}.

Other considerations relate to the optimizer's behavior in a continuous setting and how \(\beta\) affects the balance between sample efficiency and accuracy.
Theorem \ref{main_theorem} explicitly incorporates the optimizer's efficiency via a constant $\kappa_g$. 
Previous works often assume perfect AF optimization as a result of operating on a finite set.
Incorporating $\kappa_g$ makes the result more realistic with a formal acknowledgement of the practical difficulty of optimizing the AF directly impacts the achievable classification accuracy. 
A less efficient optimizer (smaller $\kappa_g$) leads to a worse classification accuracy, correctly reflecting real-world performance.
The constant \(\beta\) was introduced earlier in Algorithm \ref{alg:one} as the confidence parameter.
For the context of Theorem 4.3, demanding a higher confidence level (i.e., a larger \(\beta\)) would require more samples to achieve the same accuracy. 

Though Theorem \ref{main_theorem} is shown with the Straddle heuristic as both global and local AFs, one can use other AFs for TRLSE in practice.
We show the experimental results using other AFs in Appendix \ref{different_acq}.
In addition, although our focus is on high-dimensional spaces with the assumption of Gaussian noise, it would be interesting to extend this work and our theoretical analysis to non-Gaussian ones.
For sub-Gaussian noise assumptions, the results in Theorem \ref{main_theorem} still hold with minor modifications involving the replacement of the standard Gaussian concentration inequalities with Hoeffding-type bounds \citep{10.5555/3305381.3305469}.
Heavy-tailed noise, however, necessitates specialized design choices for GP robustness against outliers; consequently, theoretical guarantees depend on the particular robust GP formulation used.
\subsection{A Conjectured Tighter Bound}
While the preceding analysis provides a formal guarantee for TRLSE's global accuracy, its convergence rate of $O(\sqrt{\gamma_N/N})$ only accounts for the N points sampled by the global AF. 
This is a notable limitation, as $\mathcal{GP}_g$ is updated using the entire evaluation budget of $N+T$ points.
As discussed in Section \ref{sec:trust_region_reinitialization}, TRLSE's accuracy should intuitively reflect the total number of samples. 
Based on this observation, we conjecture that a tighter bound dependent on the total sample size exists. 
However, a formal proof for such a bound is highly challenging without first making a plausible, yet unproven, assumption about the behavior of the global AF at locally sampled points, which we will now discuss.
\begin{assumption}[Non-negativity of Global Acquisition Values at Local Points]\label{assumption:non_negativity_globallocal}
    The global acquisition value at $\mathbf{x}_t$, calculated using $\mathcal{GP}_g$ just before evaluating $\mathbf{x}_t$, is non-negative.
    Specifically:
    \begin{align*}
        a_g(\mathbf{x}_t) = \beta\sigma_g(\mathbf{x}_t) - |\mu_g(\mathbf{x}_t) - h| \geq 0
    \end{align*}
\end{assumption}
This assumption is grounded in the interplay between the local and global models.
$\mathcal{GP}_{t-1}^i$ is a specialized and accurate model for its TR $\mathcal{T}^i$, trained on the most relevant local data.
$\mathbf{x}_t$ is selected from $\mathcal{T}^i$ because of its high local acquisition value, meaning $\mathcal{GP}_{t-1}^i$ is uncertain about the point's classification. 
Given that this specialized local model is uncertain, it is improbable that the more general global model $\mathcal{GP}_g$ would be simultaneously highly confident in its classification of this point, which is necessary for $a_g(\mathbf{x}_t)$ to be negative. 
While such disagreements are possible, they become less frequent with more data.
\begin{figure}
    \centering
    \includegraphics[width=0.6\linewidth]{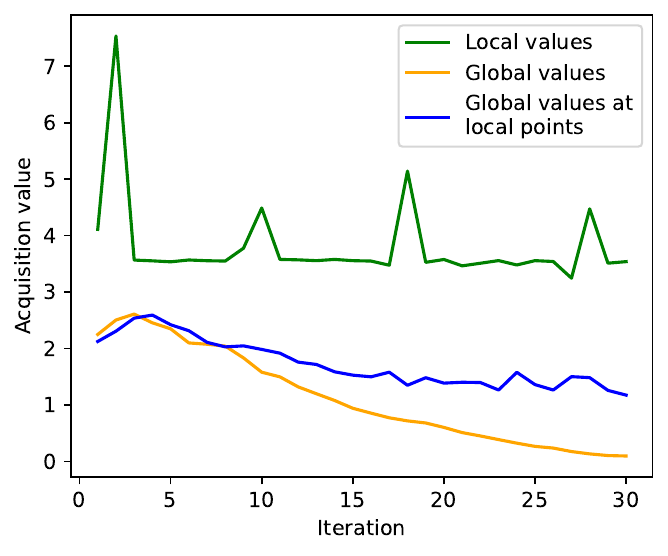} 
    \caption{An illustration of Assumption \ref{assumption:non_negativity_globallocal} on AA33D. The global acquisition values at the local points remain non-negative. Global values constantly decrease with more points sampled as shown in the proof for Theorem \ref{main_theorem}, while local values often spike when new uncertain areas are explored.}
    \label{fig:conjecture_illustration}
\end{figure}

With this assumption in place, we can explore the implications for a tighter convergence rate of TRLSE through a similar proof strategy.
Let $A^{\prime}_{N+T}$ be the running average over the global acquisition values of all $N+T$ sampled points.
$A^{\prime}_{N+T}$ can be upper bounded by a constant proportional to $\sqrt{\gamma_{N+T}/(N+T)}$ when we also have that $a_g(\mathbf{x}_t)\geq 0, \forall t$.
This tighter average can be connected to the global accuracy where two conditions are met: 
(1) the most recently sampled point is a global point,  
$\bar{\mathbf{x}}_N$, and 
(2) its acquisition value is less than or equal to the current running average, $a_g(\bar{\mathbf{x}}_N) \leq A^{\prime}_{N+T}$.
The first condition simply specifies the point in the algorithm's execution that we are analyzing. 
The second one is guaranteed to be met infinitely often, following a similar logic to that in Lemma \ref{lem:l6_infinitely_many_non_aggressive}. 
The global accuracy is then bounded by:
\begin{align}
    \epsilon \leq \frac{\beta}{\kappa_g} \sqrt{\frac{C_1 \gamma_{N+T}}{N+T}}.\label{conjectured_global_accuracy}
\end{align}
This tighter bound hinges on an assumption that is highly plausible in practice as seen in Figure \ref{fig:conjecture_illustration}.
This bound depends on the total evaluation budget $N+T$ which more accurately reflects the algorithm's operation. 
It reveals a key synergy in TRLSE's design: local exploration, driven by a separate objective, accelerates the convergence of the global model, providing a more complete theoretical explanation for the sample efficiency observed in our experiments.

%% file: 5_experiments.tex
\section{Experiments}
We evaluate TRLSE on 8 synthetic and real-world problems, comparing it against: 1) Random sampling (Random), 2) The Straddle heuristic \citep{straddle} (STR), and 3) HLSE \citep{ha2021high}.
Random sampling and STR use a single GP as their surrogate model, and we use ExpHLSE \citep{ha2021high} as the implementation of HLSE to match our problem setting of known threshold.
These baselines provide a comprehensive evaluation against naive, popular, and state-of-the-art methods.
STR is a commonly used method and has a comparable performance to other popular methods including ActiveLSE and TruVar \citep{rmile}, and its inherent adaptability to continuous spaces makes it a highly relevant benchmark for our setting.
Since STR is also highly prone to over-exploration in high-dimensional spaces, it is an ideal test for TRLSE's ability to mitigate this issue, especially when it is used as both the global and local AFs in TRLSE.
We also consider other AF choices for TRLSE including C2LSE and Thompson sampling and make explicit comparison between these methods and the corresponding TRLSE version in Appendix \ref{different_acq}.
In addition, a key focus of the comparison with HLSE is to evaluate which algorithm achieves a better trade-off between classification accuracy and computational efficiency. 
While HLSE is a state-of-the-art method, its high computational demands and hardware requirements are well-documented, so this comparison is crucial for demonstrating that TRLSE offers a more practical and scalable solution without sacrificing performance.

All experiments use search spaces scaled to \([0,1]^d\).
Methods with GPs employ a Mat\'{e}rn kernel with hyperparameter priors adjusted according to \citet{hvarfner2024vanillabayesianoptimizationperforms}.
Results using other kernels are shown in Appendix \ref{different_kernel}.
Descriptions of the benchmarks, along with their corresponding hyperparameters, are presented in Appendix \ref{detail_experiment}. 
To remove the bias of zero-mean prior by GP in classification, we standardize the function values and shift their mean to the threshold. 
We show results with TRLSE using STR as both global and local AFs.
We use \(S_1(u)=2/(1+\textup{exp(8u-6)})\) for \(S(\cdot)\) and ablate this choice in Section \ref{S_choices}.
The digit suffix in the name of each benchmark shows its dimensionality.
We report F1-scores as the measure of classification accuracy.
Each method is repeated 10 times, where the initializations are similar among all methods for each repetition.
The median and interquartile range among 10 repetitions are plotted.
Code to reproduce all experiments is available at \url{https://github.com/giang-n-ngo/TRLSE}.
\subsection{Main Results}\label{main_results}
\subsubsection{Classification performance}
\begin{figure*}[ht]
    \subfigure[Levy10]{\includegraphics[scale=0.372]{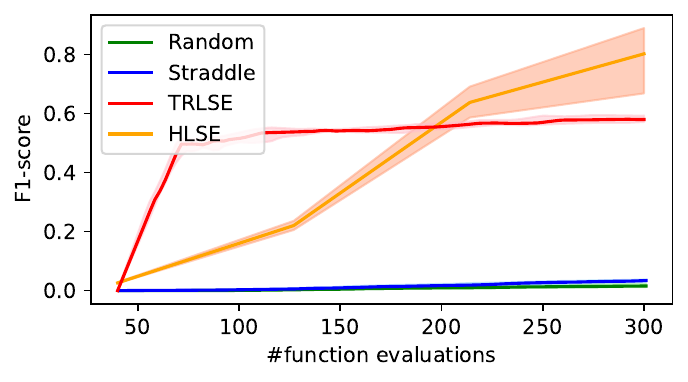}}
    \subfigure[AA33]{\includegraphics[scale=0.372]{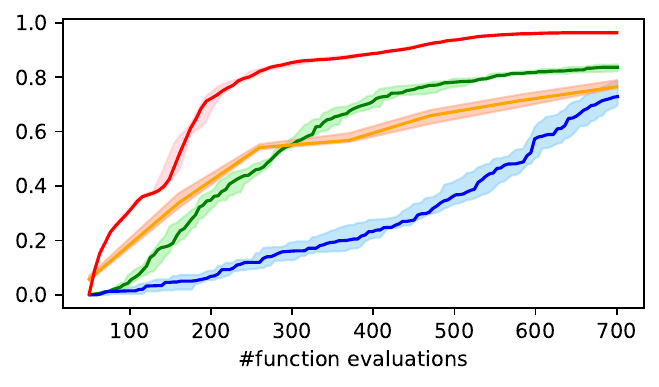}}
    \subfigure[Mazda74]{\includegraphics[scale=0.372]{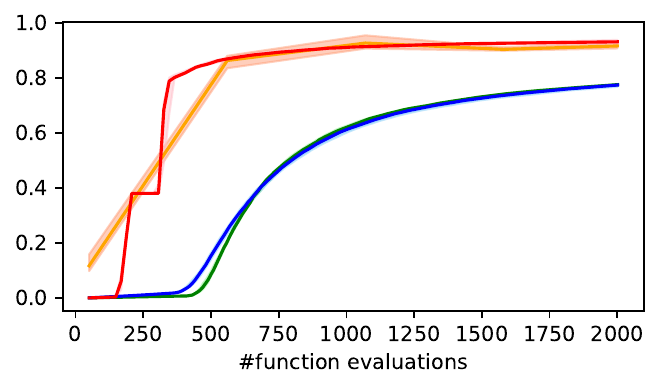}}
    \subfigure[Levy100]{\includegraphics[scale=0.372]{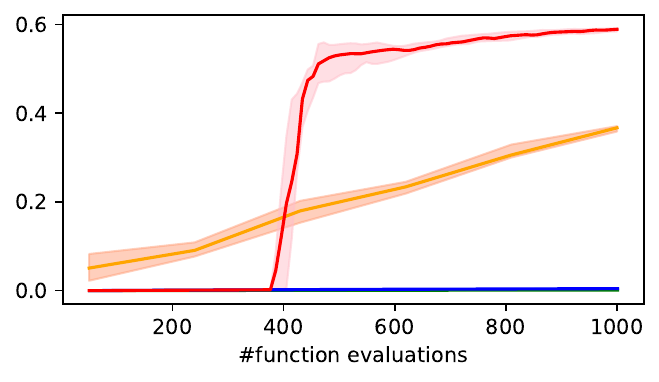}}\\
    \subfigure[Vehicle124]{\includegraphics[scale=0.372]{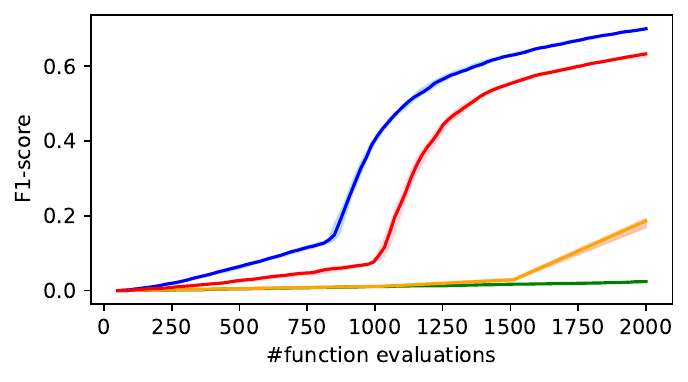}}
    \subfigure[Ackley200]{\includegraphics[scale=0.372]{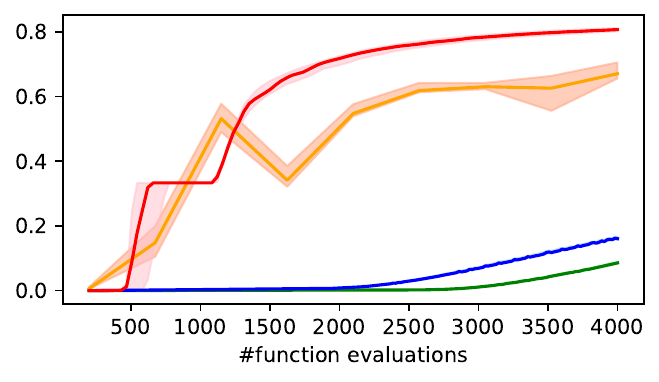}}
    \subfigure[Trid1000]{\includegraphics[scale=0.372]{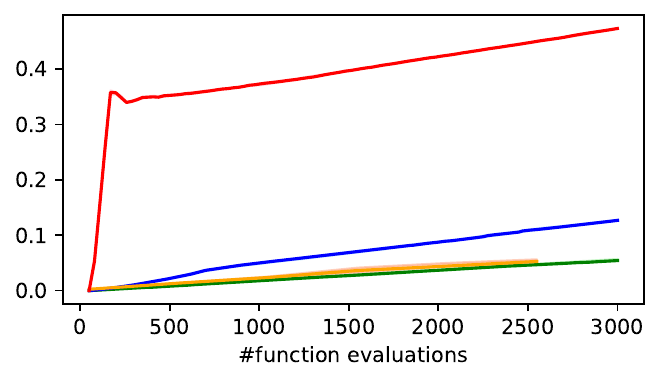}}
    \subfigure[Rosenbrock1000]{\includegraphics[scale=0.372]{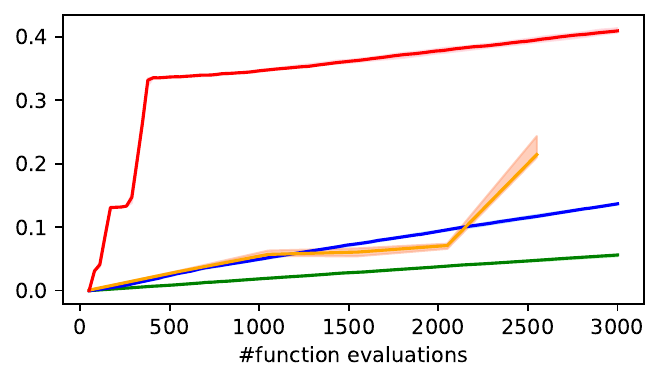}}
    \hspace*{\fill}
    \caption{Results on synthetic functions and real-world benchmark problems}
    \label{fig:main}
\end{figure*}
Figure \ref{fig:main} shows that TRLSE is at least as good as the baselines on most problems. 
TRLSE's accuracy also improves consistently throughout the process while HLSE's accuracy may fluctuate, as seen for Ackley200 (and also as reported by \citet{ha2021high}). 
HLSE also experiences out-of-memory issues and cannot complete the experiments on 1000D problems.
STR, however, struggles as a result of almost zero recall on both experiments. 
In truly high-dimensional spaces, heuristic AFs like STR tend to overly explore due to the dependence on the posterior variance which increases greatly with more dimensions. 
While STR in this case, with the help of a reduced complexity assumption using adjusted kernel hyperparameters similar to \citet{hvarfner2024vanillabayesianoptimizationperforms}, may not necessarily explore only the border of the search space (which often comes with the largest posterior variance), it becomes similar to random sampling on many experiments because the acquisition values become almost constant as a result of large posterior variance.
TRLSE reduces the effect of over-exploration by gradually updating the TRs closer to the threshold boundary and only exploring a new region if a current TR is not promising.
\subsubsection{Runtime}\label{runtime}
\begin{table}[h]
    \centering
    \caption{Average runtimes in minutes for experiments in Figure \ref{fig:main}. * indicates out-of-memory experiments.}
    \label{tab:runtime}
    \begin{tabular}{lllll}
        \toprule
        Problem & Random & STR & HLSE & TRLSE\\
        \midrule
        Levy10 & 0.1 & 0.1 & 368.2 & 4.2 \\
        AA33 & 1.0 & 1.7 & 661.4 & 11.5 \\
        Mazda74 & 6.1 & 9.8 & 599.8 & 24.4 \\
        Levy100 & 1.9 & 1.9 & 1128.1 & 24.3 \\
        Vehicle124 & 4.9 & 5.5 & 673.2 & 58.7 \\
        Ackley200 & 18.0 & 18.6 & 1568.9 & 992.2 \\
        Trid1000 & 6.3 & 8.2 & 2200.9* & 433.9 \\
        Rosen...1000 & 6.1 & 6.9 & 3867.2* & 450.7 \\
        \bottomrule
    \end{tabular}
\end{table}
Table \ref{tab:runtime} lists the average runtime for the experiments in Figure \ref{fig:main}. 
HLSE's much higher runtimes and crashes at higher dimensionalities due to memory issues are expected, given its demanding hyperparameter tuning. 
Note that HLSE ran on three Tesla V100 GPUs while other methods ran on only one. 
The runtimes of TRLSE are increasing mostly due to a combination of increasing dimensionality (more computation and more test points) and increasing number of TRs. 
Because the implementation use sequential region updates, TRLSE's runtime on Ackley200 (with 200 TRs) is much higher than those of 1000D problems (which were run with 50 TRs as explained in Appendix \ref{detail_experiment}), which means the increased number of TRs outweighs the increased dimensionality in this case.
\subsection{Ablation Studies}
Due to space constraints, we present the ablation studies in Appendix \ref{ablation_study}, with a summary of the findings below:
\begin{itemize}
    \item Setting the initial TR volume \(V_{init}\) too big or too small can reduce the performance of TRLSE, but this is only noticeable in truly high-dimensional spaces (e.g., 100D). TRLSE achieves a good local-global balance with \(V_{init}\approx0.5^d\) in practice. (Appendix \ref{ablation_v_init})
    \item Extremely small \(R\) can hurt the accuracy of TRLSE, as it limits the pool for local acquisition optimization. However, large \(R\) yields minimal benefits, as the fixed number of points selected at the local level constrains performance improvements. (Appendix \ref{ablation_R})
    \item Using random new TRs instead of the global AF for exploration significantly degrades performance in truly high-dimensional problems, highlighting the importance of informed re-initialization. (Appendix \ref{ablation_random_new_tr})
    \item Using local GPs also contributes to improved performance, especially in higher dimensions, by providing more accurate local modeling for acquisition optimization and TR updates. (Appendix \ref{ablation_one_gp})
    \item No volume update reduces performance on both 10D and 100D problems, demonstrating the necessity of adjusting TR volumes based on their proximity to the threshold boundary. In addition, TRLSE is robust to different choices of \(S(\cdot)\), as long as it is a monotonically decreasing function that maps large penalties to small adjustment factors and vice versa. Both too large and too small adjustment factors can degrade performance. (Appendix \ref{S_choices})
\end{itemize}

%% file: 6_conclusion.tex
\section{Conclusion}
This work addresses high-dimensional level set estimation using multiple local regions to capture threshold boundary.
Our algorithm, TRLSE, updates regions based on their boundary centering while using local acquisition to improve threshold boundary modeling.
Additionally, a global acquisition function spawns new regions to replace underperforming ones, ensuring global coverage.
The algorithm's theoretical classification accuracy is rigorously analyzed, ensuring an accuracy based on the number of points sampled locally and the number of regions initialized. 
Extensive experiments across multiple synthetic and real-world problems demonstrate TRLSE's superior performance against the baselines with satisfactory performance on up to 1000 dimensions.

%% file: 7_appendix_proof.tex
\section{Proof for Theoretical Results}\label{proof}
\subsection{Proof of Lemma \ref{first_lemma}}
\begin{proof}
    When \(\Bar{l}_{t-1}^i\geq h\), we know that \(\mathcal{P}_t(\mathcal{T}^i)=\Phi(|\frac{\Bar{l}_{t-1}^i+\Bar{u}_{t-1}^i-2h}{2\Bar{\sigma}_{t-1}^i}|)\geq \Phi(\beta)\geq \psi\) if the confidence level \(\beta\) is set to be higher than \(\Phi^{-1}(\psi)\). 
    It follows that the volume of \(\mathcal{T}^i\) will reduce at iteration \(t\) since \(S(\mathcal{P}_t(\mathcal{T}^i))<1\) if \(S(u)=2/(1+\textup{exp}(au-b))\) with \(b=\psi a\).

    Given that the maximum and minimum volumes possible for a TR are \(V_{max}\) and \(V_{min}/2\), the maximum number of consecutive iterations that a region has its lower-confidence bound \(\Bar{l}_{t-1}^i\geq h\) is:
    \begin{align}
        \zeta = \frac{\log(V_{init})-2\log(V_{max})}{\log(2)-\log(1+\exp(-b+a\Phi(\beta)))}
    \end{align}
    The proof when \(\Bar{u}_{t-1}^i\leq h\) is similar.
\end{proof}
\subsection{Proof of Theorem \ref{main_theorem}}
A proof for Theorem \ref{main_theorem} requires several auxiliary results. 
The first assumption is a necessary condition for the algorithm to continue sampling points, ensuring that there is always at least one point in the search space that remains uncertain with respect to the threshold $h$.
\begin{assumption}\label{assumption:non-terminated}
    At any iteration $t$ where the algorithm has not yet terminated, there exists at least one point whose classification, given the current surrogate model, remains uncertain with respect to the threshold $h$. 
    Specifically:
    \begin{itemize}
        \item Global Uncertainty: The set of points within $\bar{\mathcal{U}}_t$ whose confidence intervals from $\mathcal{GP}_g$ contains the threshold $h$ is non-empty. 
        Formally: $$\{{\mathbf{x} \in \bar{\mathcal{U}}_t \mid \mu_g(\mathbf{x}) - \beta\sigma_g(\mathbf{x}) \leq h \leq \mu_g(\mathbf{x}) + \beta\sigma_g(\mathbf{x}) }\} \neq \emptyset.$$
        \item Local Uncertainty: The set of points within $\mathcal{U}_t$ whose confidence intervals from their corresponding local Gaussian Process ($\mathcal{GP}_t^i$) contains the threshold $h$ is non-empty. 
        Formally: $$ \{\mathbf{x} \in \mathcal{U}_t \mid \exists i \in \{1, \ldots, R\} \text{ s.t. } \mathbf{x} \in \mathcal{T}^i \text{ and } \mu_i(\mathbf{x}) - \beta\sigma_i(\mathbf{x}) \leq h \leq \mu_i(\mathbf{x}) + \beta\sigma_i(\mathbf{x}) \} \neq \emptyset.$$
    \end{itemize}
\end{assumption}

Assumption \ref{assumption:non-terminated} pertains to the continued informativeness that drives the active learning process for acquiring new data points, rather than the final classification assignment.
TRLSE actively seeks out points where there is still significant uncertainty about the function's true value relative to the threshold, specifically where the confidence interval \([\mu(\mathbf{x}) - \beta\sigma(\mathbf{x}), \mu(\mathbf{x}) + \beta\sigma(\mathbf{x})]\) still contains the threshold $h$.
Therefore, Assumption \ref{assumption:non-terminated} is a practical statement that as long as the algorithm has not fully achieved its goal of confidently classifying the entire search space, there remains at least one point that is still \textit{uncertain} enough to warrant further sampling. 
In the setting with a discrete search space, this is equivalent to the termination criterion in ActiveLSE \citep{activelse} which stops the algorithm when all points are classified with sufficient confidence.
In high-dimensional continuous spaces, where fully resolving uncertainty across all points with a limited budget is highly improbable, this assumption remains valid until the problem is effectively solved and the model is sufficiently confident everywhere.
This assumption would be violated only in the extreme case where the threshold $h$ lies entirely outside the range of observed function values, causing all acquisition functions to yield negative values and rendering further active sampling trivial for TRLSE's objectives, but we do not consider such extreme scenarios in this theoretical analysis.

\begin{lemma}[Non-negativity of Acquisition Values]\label{lemma:acquisition_nonnegativity}
    It holds that $a_g(\bar{\mathbf{x}}_n) \geq 0$ and $a_l(\mathbf{x}_t) \geq 0$ for all $t \geq 1, n \geq 1$.
\end{lemma}
\begin{proof}
    According to Assumption \ref{assumption:non-terminated}, at any non-terminal iteration of the algorithm, $\bar{\mathcal{U}}_t$ is guaranteed to contain at least one point $\mathbf{x}$ s.t. $\mu_g(\mathbf{x}) - \beta\sigma_g(\mathbf{x}) \leq h \leq \mu_g(\mathbf{x}) + \beta\sigma_g(\mathbf{x})$.
    By definition, this means: $|\mu_g(\mathbf{x})-h| \leq \beta\sigma_g(\mathbf{x})$, so we can see that for this point $\mathbf{x}$, the acquisition value is non-negative: $a_g(\mathbf{x})=\beta\sigma_g(\mathbf{x})-|\mu_g(\mathbf{x})-h|\geq 0$.
    The algorithm selects the next point to evaluate, $\bar{\mathbf{x}}_n$, by maximizing its respective acquisition function over its domain (i.e. $\bar{\mathbf{x}}_n=\underset{\mathbf{x}\in\bar{U}_t}{\text{argmax }}a_g(\mathbf{x})$), so $a_g(\bar{\mathbf{x}}_n) = \underset{\mathbf{x} \in \bar{U}_t}{\text{argmax }}a_g(\mathbf{x}) \geq a_g(\mathbf{x}) \geq 0$.
    The same reasoning applies to the local acquisition function, $a_l(\mathbf{x}_t)$, thus proving that both global and local acquisition values are non-negative.
\end{proof}

\begin{lemma}[Upper bound the average global acquisition value]\label{lemma:global_acquisition_bound}
    It holds that:
    \begin{equation}
        \frac{1}{N}\sum_{n=1}^{N} a_g(\bar{\mathbf{x}}_n) \leq \sqrt{\frac{\beta^2C_1\gamma_N}{N}}
    \end{equation}
    where $\gamma_N$ is the maximum information gain over all possible sets of $N$ noisy observations.
\end{lemma}
\begin{proof}
    By definition, the global acquisition function is defined as:
    \[
        a_g(\mathbf{x}) = \beta\sigma_g(\mathbf{x}) - |\mu_g(\mathbf{x}) - h|
    \]
    We then have that:
    \begin{align} 
        \left[ \frac{1}{N}\sum_{n=1}^{N} a_g(\mathbf{\bar{x}}_n) \right]^2 & \leq \frac{1}{N} \sum_{n=1}^{N} a_g(\mathbf{\bar{x}}_n)^2 \quad \left(\text{Cauchy-Schwarz inequality}\right) \nonumber\\
        & \leq \frac{1}{N} \sum_{n=1}^{N} (\beta\sigma_g(\bar{\mathbf{x}}_n))^2 \quad \left(\text{given } 0 \leq a_g(\bar{\mathbf{x}}_n) \leq \beta\sigma_{n-1}(\bar{\mathbf{x}}_n)\right) \nonumber\\
        & \leq \frac{\beta^2C_1}{N}\sum_{n=1}^{N}\frac{log(1+\sigma^{-2}\sigma_g(\bar{\mathbf{x}}_n)^2)}{2} \quad \left(\textup{see to Lemma 5.4 in \citet{srinivas}}\right)\label{upperbound_by_variance}
    \end{align}
    Next, we will try to upper bound the right hand side of the Ineq.\ref{upperbound_by_variance} using $\gamma_N$.
    With $\sigma_g(\mathbf{x})$ being the posterior standard deviation at $\mathbf{x}$ of $\mathcal{GP}_g$ before $\mathbf{x}$ is evaluated, let us denote $\triangle I(\mathbf{x})=log(1+\sigma^{-2}\sigma_g(\mathbf{x})^2)/2$.
    From Lemma 5.3 in \citet{srinivas}, the information gain from noisy observations $\mathbf{y}_D=[y_t]_{t=1}^T\cup[\bar{y}_n]_{n=1}^N$, which is used to fit $\mathcal{GP}_g$, is given by:
    \begin{align*}
        I(\mathbf{y}_D,f)=\sum_{n=1}^{N}\triangle I(\bar{\mathbf{x}}_n) + \sum_{t=1}^{T}\triangle I(\mathbf{x}_t),
    \end{align*}
    where $\triangle I(\mathbf{x})$ can be seen as the incremental information gain from observing the noisy function evaluation at $\mathbf{x}$.
    With $\sum_{n=1}^{N}\triangle I(\bar{\mathbf{x}}_n)\leq I(\mathbf{y}_D,f)$, we can upper bound the right hand side of Ineq.\ref{upperbound_by_variance} using $\gamma_{N+T}$, the maximum information gain over all possible sets of $N+T$ noisy observations, but this is larger than our target upper bound $\gamma_N$.
    In order to obtain a tighter bound, we need to relate $\sum_{n=1}^{N}\triangle I(\bar{\mathbf{x}}_n)$ to $I([\bar{y}_n]_{n=1}^N,f)$, which is the information gain from just observing $[\bar{y}_n]_{n=1}^N$ and can be upper bounded by $\gamma_N$.

    Because $[y_t]_{t=1}^T$ and $[\bar{y}_n]_{n=1}^N$ are intermingled in $\mathbf{y}_D$ due to the ordering of observations (as can be seen in the two function evaluation steps for the global and local phases in Algorithm \ref{alg:one}), we know that $I([\bar{y}_n]_{n=1}^N,f)\neq\sum_{n=1}^{N}\triangle I(\bar{\mathbf{x}}_n)$.
    However, $I([\bar{y}_n]_{n=1}^N,f)$, which is an information gain, satisfies submodularity \citep{10.5555/3020336.3020377}.
    Since $[\bar{y}_n]_{n=1}^N$ are not necessarily the first $N$ observations in $\mathbf{y}_D$, it follows that $I([\bar{y}_n]_{n=1}^N,f)\geq\sum_{n=1}^{N}\triangle I(\bar{\mathbf{x}}_n)$.
    To put it simply, the information gain from observing $[\bar{y}_n]_{n=1}^N$ immediately at the beginning of the observation process is at least as large as the sum of the incremental information gains from observing these noisy function values later in the process.
    Therefore, we have that:
    \begin{align}
        \sum_{n=1}^{N}\triangle I(\bar{\mathbf{x}}_n) \leq I([\bar{y}_n]_{n=1}^N,f) \leq \gamma_N.\label{information_gain_bound}
    \end{align}
    Combining Ineq.\ref{upperbound_by_variance} and Ineq.\ref{information_gain_bound}, the lemma holds as a result.
\end{proof}

\begin{lemma}[Linking Acquisition Value to Classification Accuracy]\label{epsilon_acc}
    Let \(a(\mathbf{x}^*)=\underset{\mathbf{x}\in \mathcal{S}}{\textup{max }}a(\mathbf{x})\) for a set \(\mathcal{S}\) with \(\hat{\mathcal{H}}_{\mathcal{S}}\) and \(\hat{\mathcal{L}}_{\mathcal{S}}\) being the predicted superlevel and sublevel sets using the classification rules as in Algorithm \ref{alg:one}, respectively. 
    If \(a(\mathbf{x}^*)\leq \epsilon\), the probability that a point \(\mathbf{x}\in\mathcal{S}\) satisfies \(\epsilon\)-accuracy requirements is at least \(\Phi(\beta)\).
\end{lemma}
\begin{proof}
    Let \(\mathbf{x}^+\in\hat{\mathcal{H}}_{\mathcal{S}}\) and let \(\mathcal{N}(\mu(\mathbf{x}^+),\sigma^2(\mathbf{x}^+))\) be the posterior distribution used for level set classification at \(\mathbf{x}^+\).
    Given the classification rules, we know that \(\mu(\mathbf{x}^+)>h\).
    Having \(a(\mathbf{x}^*)\leq \epsilon\) means \(a(\mathbf{x})\leq\epsilon\;\forall \mathbf{x}\in\mathcal{S}\).
    This means \(\beta\sigma(\mathbf{x}^+)-|\mu(\mathbf{x}^+)-h|\leq\epsilon\), or \(\mu(\mathbf{x}^+)-\beta\sigma(\mathbf{x}^+)\geq h-\epsilon\) given \(\mu(\mathbf{x}^+)>h\) due to the classification rules.
    Conditioned on \(f(\mathbf{x}^+)\sim\mathcal{N}(\mu(\mathbf{x}^+),\sigma^2(\mathbf{x}^+))\), we have that:
    \begin{align}
        \textup{Pr}(f(\mathbf{x}^+)\leq\mu(\mathbf{x}^+)+\beta\sigma(\mathbf{x}^+))=\textup{Pr}(f(\mathbf{x}^+)\geq\mu(\mathbf{x}^+)-\beta\sigma(\mathbf{x}^+))=\Phi(\beta).\nonumber
    \end{align}
    Therefore, \(\textup{Pr}(h-f(\mathbf{x}^+)\leq\epsilon)\geq\Phi(\beta)\) if \(\mathbf{x}^+\in\hat{\mathcal{H}}_{\mathcal{S}}\). 
    Similarly, we have that \(\textup{Pr}(f(\mathbf{x}^+)-h\leq\epsilon)\geq\Phi(\beta)\) if \(\mathbf{x}^+\in\hat{\mathcal{L}}_{\mathcal{S}}\). The lemma holds as a result.
\end{proof}

With all auxiliary results established, we can now prove Theorem \ref{main_theorem}.

Let $N$ be an iteration that satisfies the theorem's condition, $a_g(\bar{\mathbf{x}}_N) \le A_N$. 
Our goal is to find a uniform accuracy bound $epsilon$ for all $\mathbf{x} \in \bar{\mathcal{U}}_N$.

According to Lemma~A.5, if we can show that $\max_{\mathbf{x} \in \bar{\mathcal{U}}_N} a_g(\mathbf{x}) \le \epsilon$, then any point in that domain is classified with $epsilon$-accuracy with probability at least $\Phi(\beta)$. 
We bound this maximum acquisition value by chaining our preceding results. 
We know that $\max_{\mathbf{x} \in \bar{\mathcal{U}}_N} a_g(\mathbf{x}) \le a_g(\bar{\mathbf{x}}_N)/\kappa_g$.
Applying the theorem's condition to the numerator, followed by the upper bound on the average from Lemma \ref{lemma:global_acquisition_bound}, we have:
$$
\frac{a_g(\bar{\mathbf{x}}_N)}{\kappa_g} \le \frac{A_N}{\kappa_g} \le \frac{1}{\kappa_g} \sqrt{\frac{\beta^2 C_1 \gamma_N}{N}}
$$
This gives a bound on the maximum acquisition value. 
Setting our accuracy bound $\epsilon$ to this value, $\epsilon = \frac{\beta}{\kappa_g} \sqrt{\frac{C_1 \gamma_N}{N}}$, satisfies the condition for Lemma \ref{epsilon_acc}.

\subsection{Proof of Lemma \ref{lem:l6_infinitely_many_non_aggressive}}
\begin{proof}[Proof by Contradiction]
Let $B_N = \sqrt{\frac{\beta^2C_1\gamma_N}{N}}$.
From Lemma \ref{lemma:global_acquisition_bound}, we know that $A_N\leq B_N$, and from Lemma \ref{lemma:acquisition_nonnegativity}, we know that $A_N \ge 0$ for all $N$.
Assuming common kernels with their corresponding maximum information gain as stated in Theorem 5 in \citet{srinivas}, we have that $B_N$ converges to 0 as $N \to \infty$.
By the Squeeze Theorem, this implies that $\lim_{N\to\infty} A_N = 0$.

Now, assume for contradiction that the set $I = \{N \in \mathbb{N} \mid a_g(\bar{\mathbf{x}}_N) \le A_N\}$ is finite. This means there must exist some integer $K_0$ such that for all $N > K_0$, the opposite condition holds:
$$
a_g(\bar{\mathbf{x}}_N) > A_N
$$
We can express the running average recursively as $N A_N = (N-1)A_{N-1} + a_g(\bar{\mathbf{x}}_N)$. Substituting our assumption for $N > K_0$:
\begin{align*}
    N A_N &> (N-1)A_{N-1} + A_N \\
    (N-1)A_N &> (N-1)A_{N-1} \\
    A_N &> A_{N-1}
\end{align*}
This shows that for all $N > K_0$, the sequence $\{A_N\}$ is strictly increasing. Furthermore, since all $a_g(\cdot)$ are non-negative, $A_N > 0$ for all $N$.

We have thus established that $\{A_N\}_{N > K_0}$ is a strictly increasing sequence of positive numbers. However, an increasing sequence of positive numbers cannot converge to 0. This contradicts our initial finding that $\lim_{N\to\infty} A_N = 0$.

Therefore, our initial assumption must be false, and the set $I$ must be infinite.
\end{proof}

%% file: 7_appendix_extra_exp.tex
\section{Details on Experiments}\label{detail_experiment}
\subsection{Experiment Setup}
\begin{table}
    \centering
    \begin{tabular}{llclll}
        \toprule
        Problem (Source) & Domain & \% of superlevel set & \(V_{init}\) & \(V_{max}\) & \(R\) \\ 
        \midrule
        MC2D \citep{c2lse} & $[0,9]^2$ & 12\% & \(10^{-3}\) & \(5\times10^{-2}\) & 10 \\ 
        Mishra03 \citep{aplse} & $[-5,5]^2$ & 61.5\% & \(10^{-4}\) & \(5\times10^{-2}\) & 10 \\
        Levy10 \citep{ha2021high} & $[-10,10]^{10}$ & 20\% & \(10^{-5}\) & \(10^{-1}\) & \(40\) \\
        AA33 \citep{ha2021high} & See source paper & 20\% & \(10^{-4}\) & \(10^{-1}\) & 50 \\
        Mazda74 \citep{mazda74} & See source paper & 24.35\% & \(10^{-20}\) & \(10^{-2}\) & 50 \\
        Levy100 (Extension of Levy10) & $[-10,10]^{100}$ & 20\% & \(10^{-30}\) & \(10^{-2}\) & 50 \\ 
        Vehicle124 \citep{anjos2008mopta} & See source paper & 3.3\% & \(10^{-35}\) & \(10^{-2}\) & 50 \\
        Ackley200 (Extension of Ackley10) & $[-5,10]^{200}$ & 20\% & \(10^{-60}\) & \(10^{-2}\) & 200 \\
        Trid1000 \\\citep{surjanovic2013virtual} & $[-10^6,10^6]^{1000}$ & 20\% & \(10^{-300}\) & \(10^{-2}\) & 50 \\
        Rosenbrock1000 \\\citep{surjanovic2013virtual} & $[-5,10]^{1000}$ & 20\% & \(10^{-300}\) & \(10^{-2}\) & 50 \\
        \bottomrule
    \end{tabular}
    \caption{TRLSE's data sources and hyperparameters}
    \label{tab:hyperparameters}
\end{table}
Table \ref{tab:hyperparameters} lists the hyperparameters of TRLSE for all synthetic and real-world benchmark problems used in this paper. 
As discussed earlier in Section \ref{ablation_v_init}, \(V_{max}\) is often set to be from 1-10\% of the search space's volume, much larger than \(V_{init}\), which is usualy \(0.5^d\), only to stop a TR from getting infinitely large.
The number of regions \(R\) maintained simultaneously is also often quite similar between the problems, regardless of the dimensionality of the underlying function.
Although Lemma \ref{first_lemma} only requires the confidence parameter \(\beta\) to be higher than \(\Phi^{-1}(0.75)(\approx0.6745)\), we keep \(\beta\) fixed at 1.96 to demand higher confidence for all experiments.
We originally ran the 1000D problems with \(R\geq 200\) (at least on par with the number of regions for Ackley200), but prohibitively slow runtimes and memory issues prevent us from setting \(R\) higher for these problems, which require much more computational resources with their higher number of dimensions.

Details about Levy10, Ackley10, Alpine10, Protein20, and AA33 are described in \citet{ha2021high}. 
For Mazda74, we first generate 100,000 uniformly random input points, with 74 dimensions for each point, within the allowed input ranges for SUV (as specified by \citet{mazda74}).
We then run the approximation code (available at https://ladse.eng.isas.jaxa.jp/benchmark/) over these points to obtain the side impact measurement of the SUV, which is used as the output value.
We use 0 as the threshold for Mazda74 to align with the constraint for side impact shown by \citet{mazda74}.
Vehicle124 is originally called MOPTA08 by \citet{anjos2008mopta}, which is a minimization problem with hard constraints.
We follow the practice of \citet{eriksson2021high} to obtain an objective value using soft constraints. 
The specific processing code was taken from the source code in \citet{papenmeier2022increasing}.
The threshold for this problem was determined by raising it until random sampling achieves near zero in F1-score.

In practice, initializing TR with \(V_{init}=0.5^d\) can lead to potential floating-point precision issues.
In fact, the smallest value for \(V_{init}\) shown in Table \ref{tab:hyperparameters} is \(10^{-300}\) for the 1000D problems, which nearly reaches the limit of double-precision floating-point representation.
To test the effect of floating-point precision, we compare using power and multiplication operations with logarithmic operations. 
We observe up to \(10^{-11}\%\) difference in volume calculations and \(10^{-13}\%\) difference in TR side length calculations, with higher differences for higher dimensions, being highest in the 1000D case. 
These differences are negligible in practice, given the lengths of around 0.5.
To prevent potential floating-point precision issues, future implementation should consider using logarithmic operations for volume and side length calculations.
\subsection{The Effect of Random Sampling on Evaluation Results}
Our test set for each problem comprised of 100,000 points randomly sampled across the input space.
While random sampling may not yield the most comprehensive test set for evaluation, in high-dimensional spaces without prior knowledge of the function, it is the best approach to test classification accuracy throughout the entire input space. 
The alternative of using a grid-based test set is simply infeasible in high-dimensional spaces. 
That said, we have tested the effect of random sampling on the evaluation results of TRLSE at the final iteration by varying the size of the test set from 10,000 to 100,000,000 points (and repeating 10 times for each size) on some benchmarks.
From Table \ref{tab:random_sampling}, although increasing the test set size does lead to more stable F1-scores, it is clear that the results reported in Figure \ref{fig:main} are already quite close to the results with much larger test sets. 
\begin{table}
    \centering
    \begin{tabular}{llll}
        \toprule
        Test size & Mazda74 & Ackley200 & Rosenbrock1000 \\ 
        \midrule
        10000 & 0.8677 ($\pm$0.003499) & 0.8379 ($\pm$0.004048) & 0.4202 ($\pm$0.005498) \\
        100000 & 0.8686 ($\pm$0.002103) & 0.8379 ($\pm$0.001500) & 0.4185 ($\pm$0.001230) \\
        1000000	& 0.8684 ($\pm$0.000282) & 0.8379 ($\pm$0.000469) & 0.4180 ($\pm$0.000461) \\
        10000000 & 0.8685 ($\pm$0.000196) & 0.8380 ($\pm$0.000193) & 0.4179 ($\pm$0.000148) \\
        100000000 & 0.8686 ($\pm$0.000050) & 0.8380 ($\pm$0.000050) & 0.4178 ($\pm$0.000056) \\
        \bottomrule
    \end{tabular}
    \caption{Resulting F1-scores with different test sizes.}
    \label{tab:random_sampling}
\end{table}
\section{Statistical Tests}\label{stat_test}
We perform statistical tests to compare the performance of TRLSE with the baselines, previously shown in Figure \ref{fig:main}, based on two aggregated metrics derived from each run: (1) the mean F1-score across all iterations, and (2) the final F1-score at the last iteration.
For each of these two metrics, following the procedure suggested by \citet{JMLR:v7:demsar06a}, we first use the Friedman test to assess the significance of the differences in performance across the different methods on each benchmark problem.
If the Friedman test indicates significant differences, we conduct post-hoc analysis using pairwise Wilcoxon signed-ranked tests to identify specific differences, with p-values adjusted via the Holm-Bonferroni method.
For any pair found to have a statistically significant difference, we then compare their respective means across all repetitions to identify the superior method.
A significance level of 0.05 is used for all tests.
\begin{table*}
    \centering
    \caption{Results of the statistical tests to compare TRLSE with the baselines (with mean F1-score across all iterations).}
    \label{tab:stat_results_final_mean}
    \begin{tabular}{lccccc}
    \toprule
    \multirow{3}{*}{\begin{tabular}[c]{@{}c@{}}Benchmark\\ Problem\end{tabular}} & 
    \multirow{3}{*}{\begin{tabular}[c]{@{}c@{}}Friedman\\ test's\\ p-value\end{tabular}} & 
    \multicolumn{3}{c}{\multirow{2}{*}{\begin{tabular}[c]{@{}c@{}}Wilcoxon signed-ranked\\ test's p-value between TRLSE and\end{tabular}}} & 
    \multirow{3}{*}{\begin{tabular}[c]{@{}c@{}}Best\\ Algorithm\end{tabular}} \\

    & & \multicolumn{3}{c}{} & \\ 
    \cmidrule(lr){3-5}

    & & Random & Straddle & HLSE & \\
    \midrule

    Levy10 & $<$0.001 & 0.0117  & 0.0117 & 0.0117 & TRLSE \\
    AA33 & $<$0.001 & 0.0117  & 0.0117 & 0.0117 & TRLSE \\
    Mazda74 & $<$0.001 & 0.0117  & 0.0117 & 0.0391 & HLSE \\
    Levy100 & $<$0.001 & 0.0117  & 0.0117 & 0.0117 & TRLSE \\
    Vehicle124 & $<$0.001 & 0.0117  & 0.0117 & 0.0117 & Straddle \\
    Ackley200 & $<$0.001 & 0.0117  & 0.0117 & 0.0117 & TRLSE \\
    Trid1000 & $<$0.001 & 0.0117  & 0.0117 & 0.0117 & TRLSE \\
    Rosenbrock1000 & $<$0.001 & 0.0117  & 0.0117 & 0.0117 & TRLSE \\
    \bottomrule
    \end{tabular}
\end{table*}

\begin{table*}[t]
    \centering
    \caption{Results of the statistical tests to compare TRLSE with the baselines (with final F1-score at last iteration).
    TRLSE has a higher mean than the HLSE on Mazda74, but the difference is not statistically significant.}
    \label{tab:stat_results_final_last}
    \begin{tabular}{lccccc}
    \toprule
    \multirow{3}{*}{\begin{tabular}[c]{@{}c@{}}Benchmark\\ Problem\end{tabular}} & 
    \multirow{3}{*}{\begin{tabular}[c]{@{}c@{}}Friedman\\ test's\\ p-value\end{tabular}} & 
    \multicolumn{3}{c}{\multirow{2}{*}{\begin{tabular}[c]{@{}c@{}}Wilcoxon signed-ranked\\ test's p-value between TRLSE and\end{tabular}}} & 
    \multirow{3}{*}{\begin{tabular}[c]{@{}c@{}}Best\\ Algorithm\end{tabular}} \\

    & & \multicolumn{3}{c}{} & \\ 
    \cmidrule(lr){3-5}

    & & Random & Straddle & HLSE & \\
    \midrule

    Levy10 & $<$0.001 & 0.0117  & 0.0117 & 0.0117 & HLSE \\
    AA33 & $<$0.001 & 0.0117  & 0.0117 & 0.0117 & TRLSE \\
    Mazda74 & $<$0.001 & 0.0117  & 0.0117 & 0.4316 & TRLSE/HLSE \\
    Levy100 & $<$0.001 & 0.0117  & 0.0117 & 0.0117 & TRLSE \\
    Vehicle124 & $<$0.001 & 0.0117  & 0.0117 & 0.0117 & Straddle \\
    Ackley200 & $<$0.001 & 0.0117  & 0.0117 & 0.0117 & TRLSE \\
    Trid1000 & $<$0.001 & 0.0117  & 0.0117 & 0.0117 & TRLSE \\
    Rosenbrock1000 & $<$0.001 & 0.0117  & 0.0117 & 0.0117 & TRLSE \\
    \bottomrule
    \end{tabular}
\end{table*}
Tables \ref{tab:stat_results_final_mean} and \ref{tab:stat_results_final_last} present the results of the statistical tests comparing TRLSE with the baselines, using the mean and final F1-scores, respectively. 
The Friedman test's p-values indicate significant differences in performance across the methods for all benchmark problems.
With the mean F1-score as the aggregated performance metric, the Wilcoxon signed-rank test's p-values reveal that TRLSE significantly outperforms the baselines on 6/8 problems.
With the final F1-score as the aggregated performance metric, we observe that TRLSE significantly outperforms the baselines on 5/8 problems, with a tie with HLSE on Mazda74.
This confirms the superiority of TRLSE over the baselines in terms of both mean and final F1-scores.
\section{Ablation Studies}\label{ablation_study}
\subsection{Trust Region Volume}\label{ablation_v_init}
While \(V_{max}\) is often set extremely large at 1-10\% of the total search volume simply as a stopping criterion for the TRs, \(V_{init}\) directly affects the performance of TRLSE as all TRs are initialized at this volume. 
Figures \ref{fig:ablation_levy10_v_init} and \ref{fig:ablation_levy100_v_init} indicate that setting \(V_{init}\) either too big or too small can reduce the performance of the model but is only noticeable in a truly high-dimensional space.
Following experimental results, TRLSE generally works best with \(V_{init}\approx0.5^d\) in practice, striking a local-global balance.
When the TRs are initialized too big, local AFs cannot be optimized accurately. 
When initialized too small, they need more iterations to expand, and the pool for local acquisition optimization is also small, hindering the detection of the threshold boundary at the local level. 
\begin{figure}
    \centering
    \subfigure[Levy10 - \(V_{init}\)\label{fig:ablation_levy10_v_init}]{\includegraphics[scale=0.7]{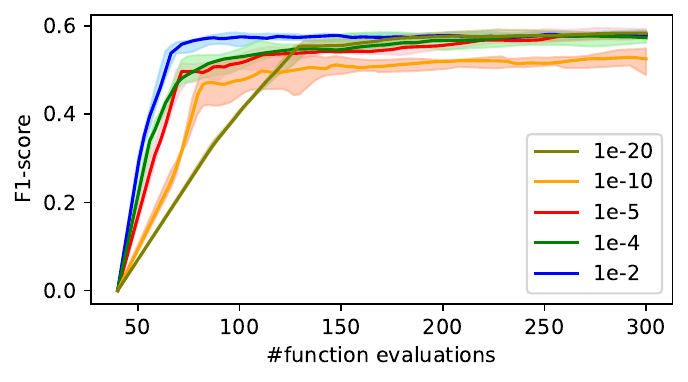}}
    \subfigure[Levy100 - \(V_{init}\)\label{fig:ablation_levy100_v_init}]{\includegraphics[scale=0.7]{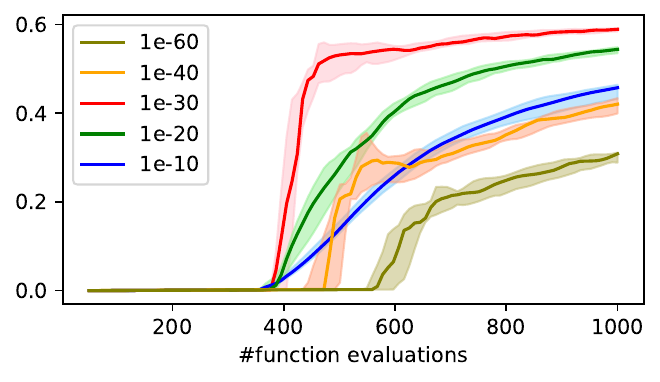}}\\
    \subfigure[Levy10 - \(R\)\label{fig:ablation_levy10_C}]{\includegraphics[scale=0.7]{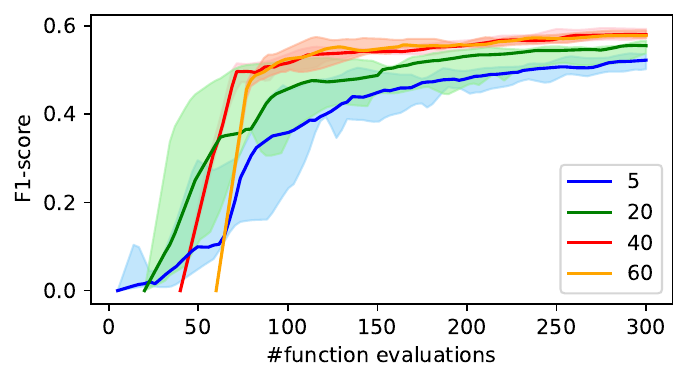}}
    \subfigure[Levy100 - \(R\)\label{fig:ablation_levy100_C}]{\includegraphics[scale=0.7]{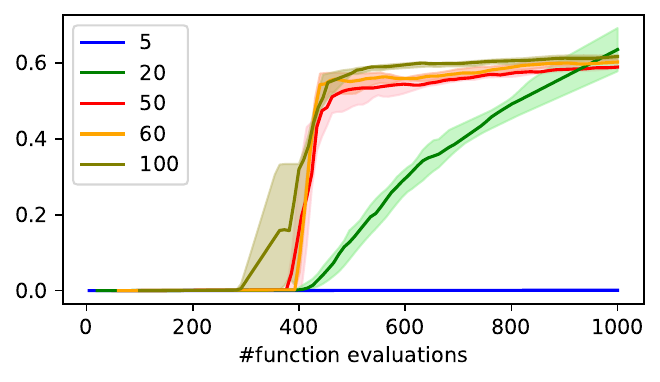}}
    \caption{Ablation studies varying \(V_{init}\) and \(R\)}
    \label{fig:ablation_1}
\end{figure}
\subsection{Number of Simultaneous Regions \texorpdfstring{$R$}{R}}\label{ablation_R}
Figures \ref{fig:ablation_levy10_C} and \ref{fig:ablation_levy100_C} show TRLSE has minimal benefits from an increase of \(R\) and that an extremely small \(R\) can hurt the accuracy of TRLSE. 
Smaller \(R\) means that the pool for local acquisition optimization is also smaller. 
Thus, the point selected by the local AF can be uninformative if \(R\) is small. 
On the other hand, a fixed number of points selected at the local level means that an ever-increasing local pool can only select that many informative data points within the TRs, so performance is not significantly improved.
\subsection{Region Re-initialization}\label{ablation_random_new_tr}
Using the global acquisition function, TRLSE aims to re-initialize underperforming TRs at more informative locations. 
We ablate this design choice by replacing it with random re-initialization where a new TR is spawned at a random location in the search space (still outside the existing TRs).
In Figure \ref{fig:ablation_2}, no clear improvement is seen on a 10D problem when TRLSE uses the global AF, but the difference widens significantly on a 100D problem.
This is expected since the search for informative regions is more challenging on search spaces with higher dimensionality where random re-initialization becomes ineffective.
\begin{figure}[ht]
    \centering
    \subfigure[Levy10\label{fig:ablation_levy10_random_onegp}]{\includegraphics[scale=0.7]{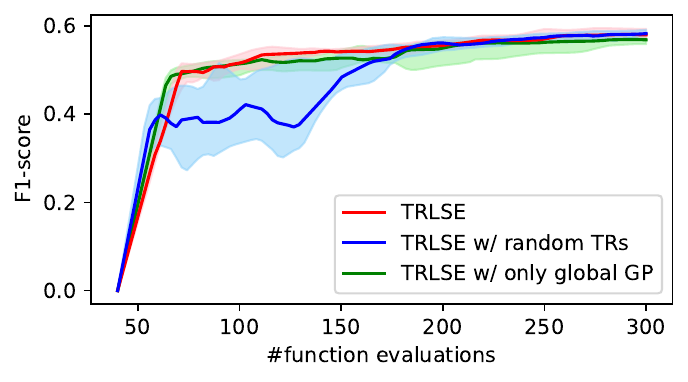}}
    \subfigure[Levy100\label{fig:ablation_levy100_random_onegp}]{\includegraphics[scale=0.7]{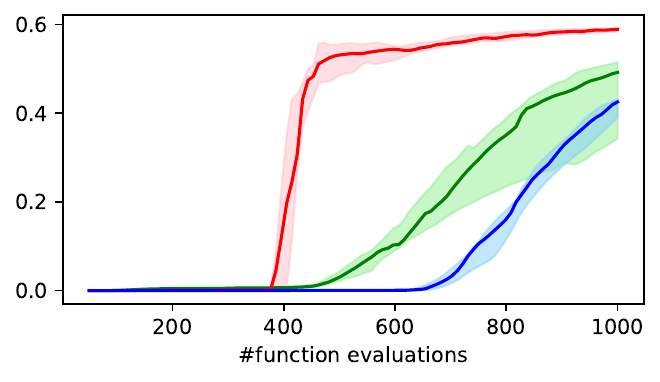}}
    \caption{Ablation studies with: 1) random new TRs, and 2) only the global GP for computations.}
    \label{fig:ablation_2}
\end{figure}
\subsection{Using Local GPs}\label{ablation_one_gp}
We verify the necessity of using local GPs for local acquisition optimization and TR updates using a modified version of TRLSE with only the global GP for local computation. 
Figure \ref{fig:ablation_2} again shows that the difference is small on the 10D problem but becomes significant on the 100D problem, which means using local GPs clearly improves TRLSE's performance. 
This can be explained by a fact demonstrated in Section 3.6 of the TuRBO paper by \citet{NEURIPS2019_6c990b7a}: local GPs, with better fitted hyperparameters, model the local regions better than the global GP. 
More accurate local modeling improves the accuracy of the local AF, which results in better local sampling and eventually better refinement for the boundary region.
\subsection{Other Choices for \texorpdfstring{$S$}{S}}\label{S_choices}
We explore other choices for \(S(\cdot)\), which converts the penalty \(\mathcal{P}_t\) to the adjustment factor for updating the volume of a TR. 
Firstly, Figures \ref{fig:ablation_levy10_S_1} and \ref{fig:ablation_levy100_S_1} show that maintaining constant TR volume (i.e., \(S(u)=1,\forall u\)) degrades TRLSE's performance.
This demonstrates the necessity of updating region volume according to how it centers at the threshold boundary.

Next, we modify \(S_1\) to still take the sigmoid form but output different ranges for the adjustment factor.
Figures \ref{fig:ablation_levy10_S_1} and \ref{fig:ablation_levy100_S_1} show that \(S_1\) and \(S_3\) have similar performances, which aligns with the similarity of the two functions compared to \(S_2\) and \(S_4\).
Both small (i.e. \(S_4\)) and large factors (i.e. \(S_2\)) lead to a minor decrease in performance, though they still converge to the same results.
Slow TR updates (i.e. small adjustment factors) retain underperforming TRs longer than necessary before getting discarded.
With too fast TR updates (i.e. large adjustment factors), premature discards may limit the ability to further explore locally at new promising regions.
However, this also indicates that TRLSE is robust to a wide range of adjustment factors.

Finally, we tried some linear functions for \(S(\cdot)\).
The gap is not significant in Figure \ref{fig:ablation_levy10_linear}, but in \ref{fig:ablation_levy100_linear}, both too slow and too fast region updates (i.e. \(S_5\) and \(S_7\)) again lead to slightly worse performance than \(S_6\), which has a similar ranges compared to \(S_1\).
However, the overall observation is that a linear form for \(S(\cdot)\) performs almost as well as the sigmoid form, indicating that \(S(u)\) only needs to be smoothly decreasing from around 2 to 0 as \(u\) increases from 0.5 to 1.
\begin{figure}
    \centering
    \subfigure[Levy10\label{fig:ablation_levy10_S_1}]{\includegraphics[scale=0.7]{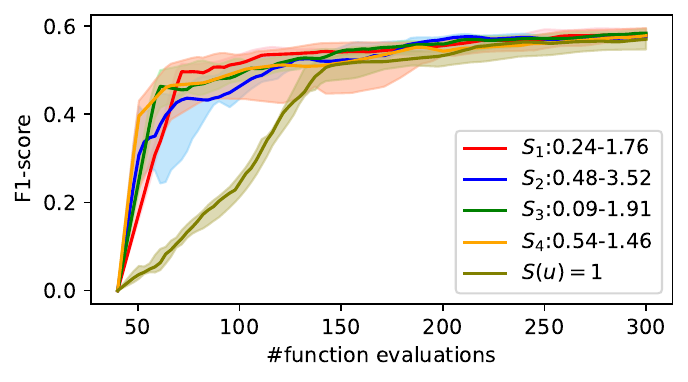}}
    \subfigure[Levy100\label{fig:ablation_levy100_S_1}]{\includegraphics[scale=0.7]{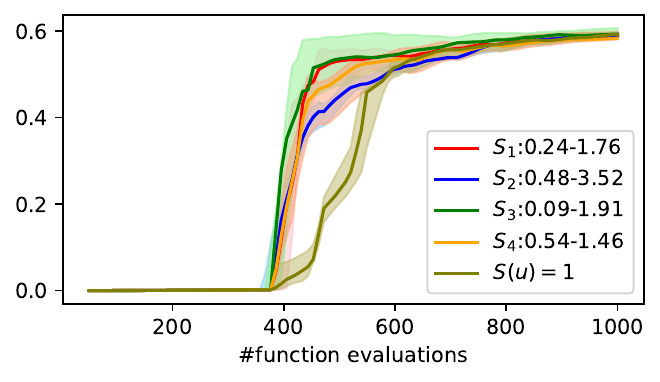}}\\
    \subfigure[Levy10\label{fig:ablation_levy10_linear}]{\includegraphics[scale=0.7]{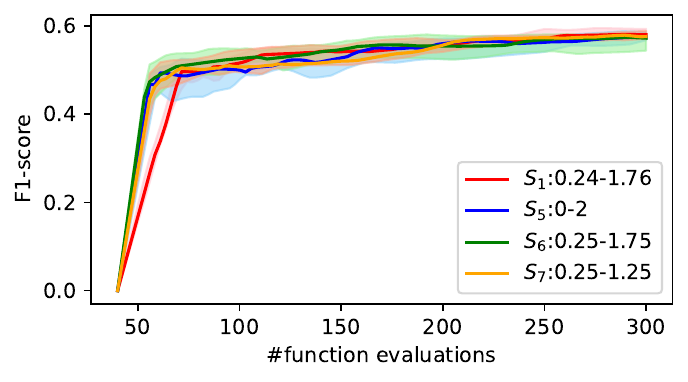}}
    \subfigure[Levy100\label{fig:ablation_levy100_linear}]{\includegraphics[scale=0.7]{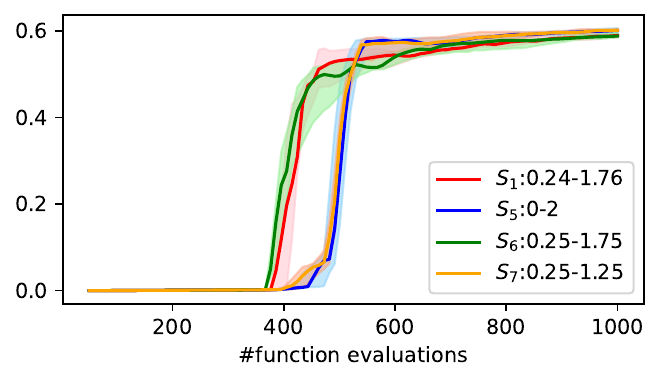}}
    \caption{Ablation studies for \(S(\cdot)\): (a,b) No volume update + sigmoid variants, and (c,d) linear variants.}
    \label{fig:ablation_3}
\end{figure}
\section{Additional Results}
\subsection{Using Different Acquisition Functions}\label{different_acq}
Besides Straddle, we also evaluate TRLSE when it uses other functions for the global and local acquisition functions to test the robustness of the algorithm w.r.t different types of acquisition functions. 
We employ two acquisition functions, namely C2LSE \citep{c2lse} and Thompson Sampling (TS). 
Although popularly adopted in BO, TS has never been used in the context of LSE. 
In this experiment, we define the TS acquisition function for LSE as \(a_{TS}(\mathbf{x})=|\Tilde{f}(\mathbf{x})-h|\) with \(\Tilde{f}(\mathbf{x})\sim \mathcal{N}(\mu(\mathbf{x}),\sigma^2(\mathbf{x}))\) where \(\mathcal{N}(\mu(\mathbf{x}),\sigma^2(\mathbf{x}))\) is the posterior distribution at \(\mathbf{x}\) given by either current global or local GP (depending on where the TS acquisition function is optimized on).

As seen in Figures \ref{fig:c2lse} and \ref{fig:TS}, TRLSE maintains a robust improvement over the baselines for both acquisition functions on most problems. 
An additional observation is that TS exhibits a comparable performance with the other two acquisition functions, making it a viable option for the LSE context.
\begin{figure*}[ht]
    \centering
    \subfigure[Levy10]{\includegraphics[scale=0.37]{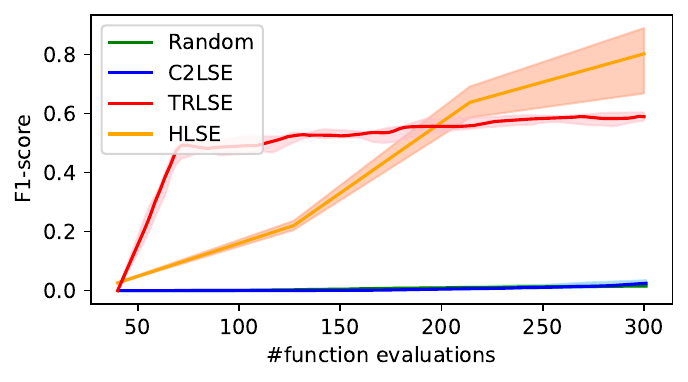}}
    \subfigure[AA33]{\includegraphics[scale=0.37]{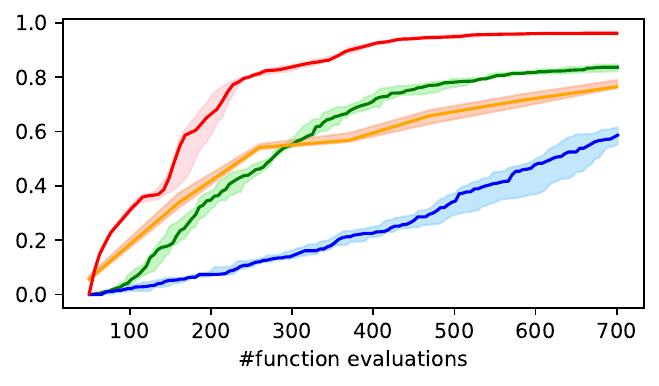}}
    \subfigure[Mazda74]{\includegraphics[scale=0.37]{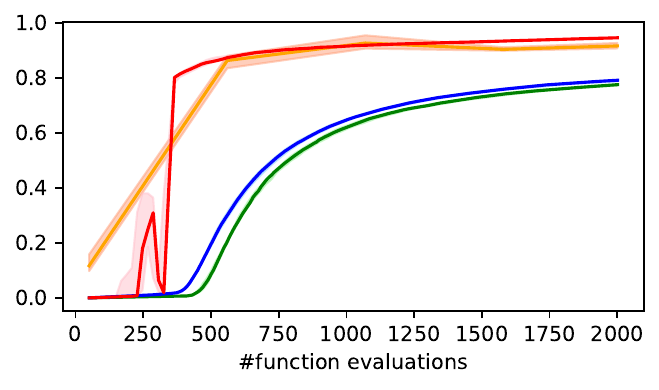}}
    \subfigure[Levy100]{\includegraphics[scale=0.37]{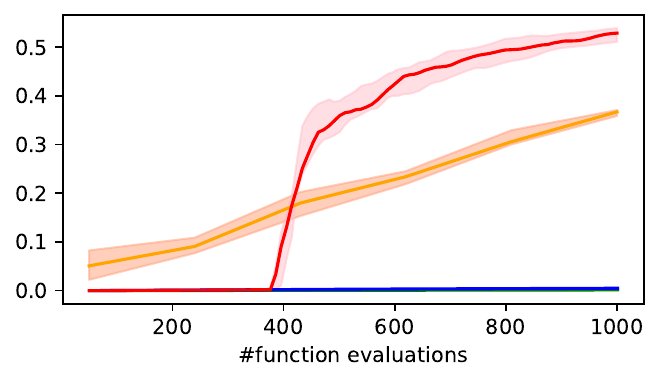}}\\
    \subfigure[Vehicle124]{\includegraphics[scale=0.37]{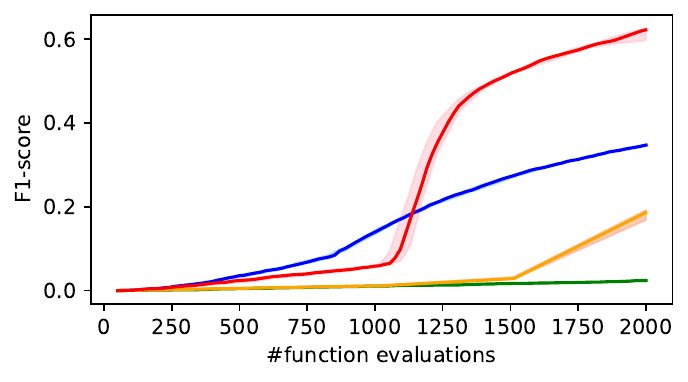}}
    \subfigure[Ackley200]{\includegraphics[scale=0.37]{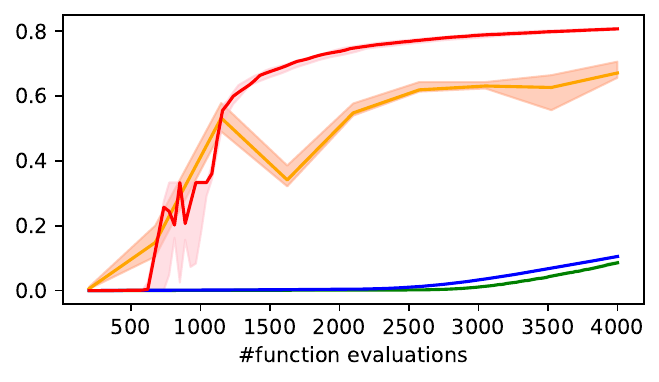}}
    \subfigure[Trid1000]{\includegraphics[scale=0.37]{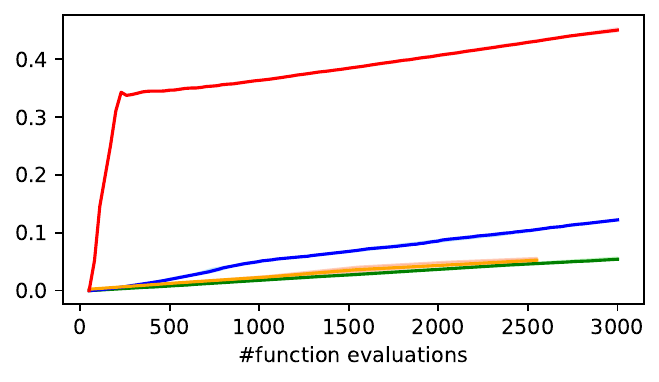}}
    \subfigure[Rosenbrock1000]{\includegraphics[scale=0.37]{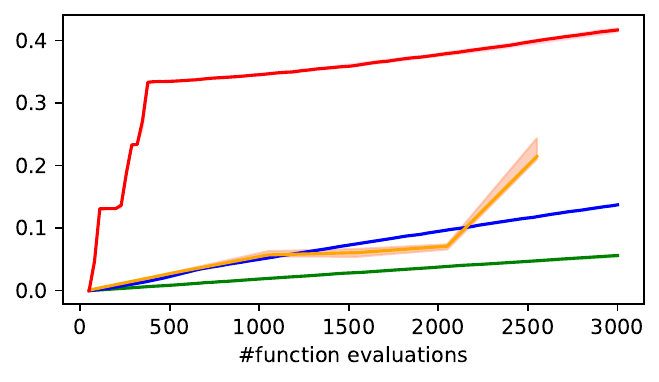}}
    \caption{TRLSE using C2LSE as both global and local acquisition functions}
    \label{fig:c2lse} 
\end{figure*}
\begin{figure*}[ht]
    \centering
    \hspace*{\fill}
    \subfigure[Levy10]{\includegraphics[scale=0.37]{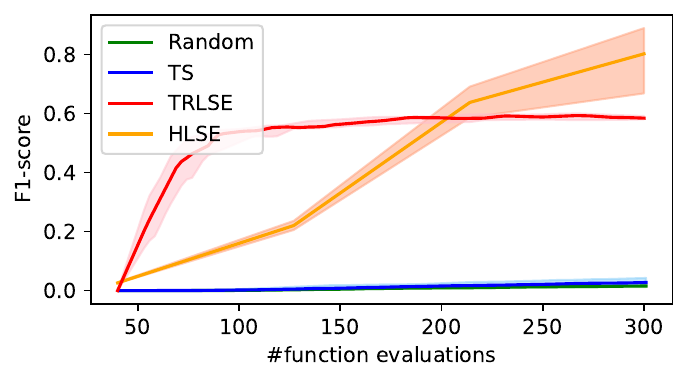}}
    \subfigure[AA33]{\includegraphics[scale=0.37]{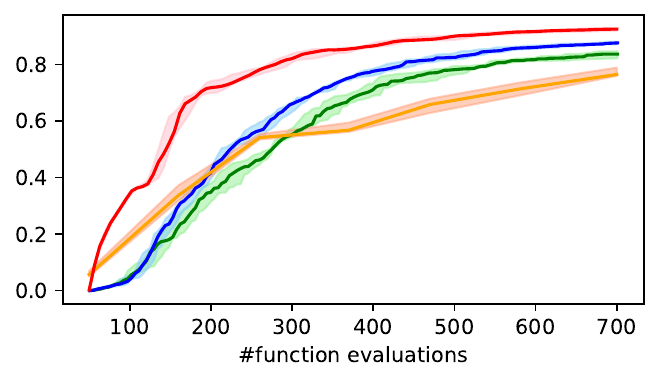}}
    \subfigure[Mazda74]{\includegraphics[scale=0.37]{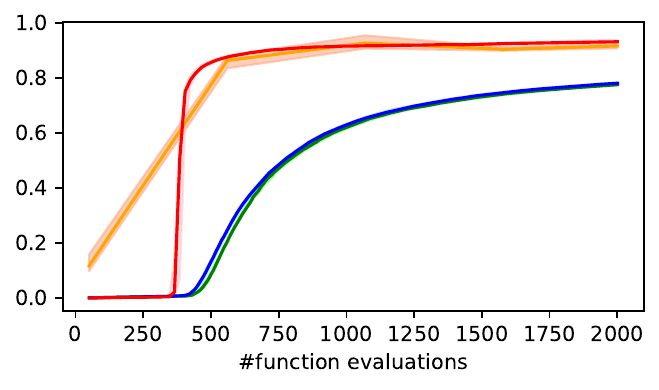}}
    \subfigure[Levy100]{\includegraphics[scale=0.37]{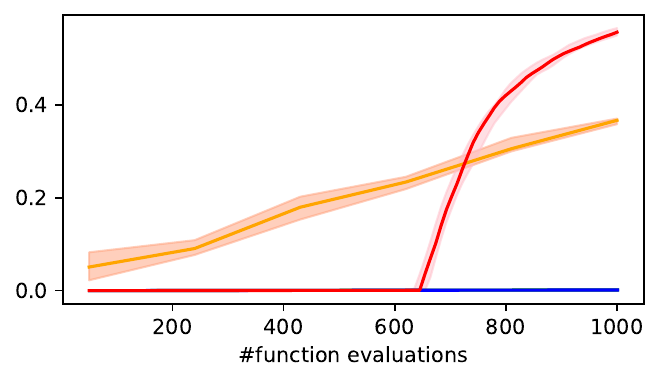}}
    \hspace*{\fill}
    \subfigure[Vehicle124]{\includegraphics[scale=0.37]{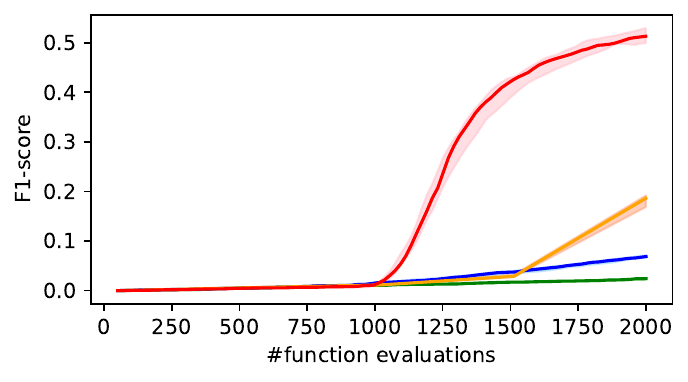}}
    \subfigure[Ackley200]{\includegraphics[scale=0.37]{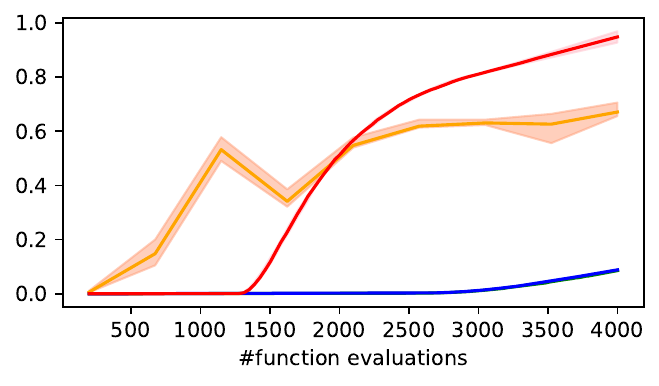}}
    \subfigure[Trid1000]{\includegraphics[scale=0.37]{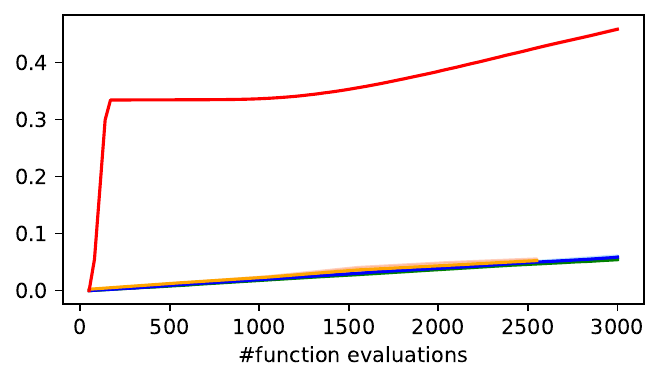}}
    \subfigure[Rosenbrock1000]{\includegraphics[scale=0.37]{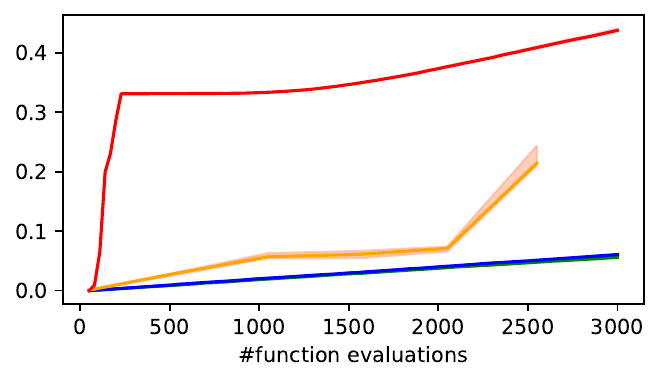}}
    \caption{TRLSE using TS as both global and local acquisition functions}
    \label{fig:TS} 
  \end{figure*}
\subsection{Low-dimensional Problems and Sampling Behavior}
\begin{figure}[h!]
    \centering
    \subfigure[MC2D]{\includegraphics[scale=0.6]{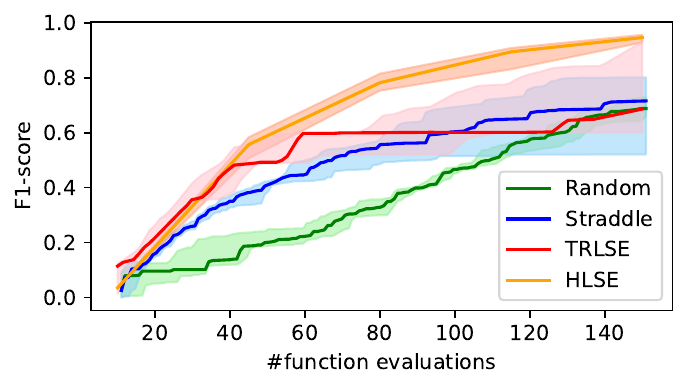}}\quad
    \subfigure[Mishra03]{\includegraphics[scale=0.6]{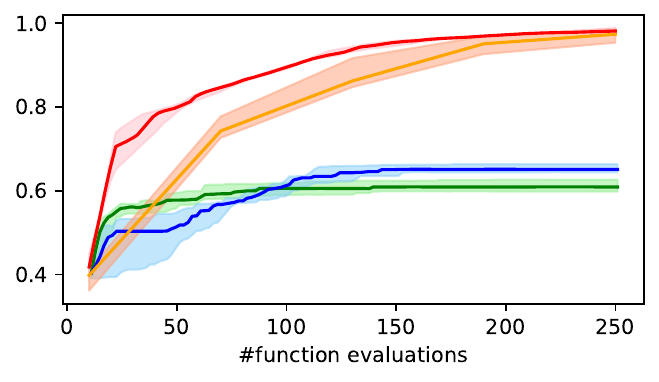}}
    \caption{Results on low-dimensional synthetic functions}
    \label{fig:low}
\end{figure}
Despite not being specifically tailored for low-dimensional problems, TRLSE still demonstrates a competitive accuracy compared to the baselines on two synthetic functions MC2D \citep{c2lse} and Mishra03 \citep{aplse}, as can be seen in Figure \ref{fig:low}. 
These results suggest that TRLSE is a robust algorithm suitable for most LSE algorithms regardless of the dimensionality of the underlying function.
Contrary to the claim that HLSE performs worse than GP-based methods in the low-dimensional settings by \citet{ha2021high}, it actually exhibits significantly better performances than Straddle on both MC2D and Mishra03.
We believe this claim was from the results on Branin (d=2) and Hartman3 (d=3) functions, which are not good benchmarks for LSE algorithms due to the fact that Straddle only needs a few data points for high F1-scores for both functions.
\subsection{Different Kernels}\label{different_kernel}
We ablate the kernel choice for GP modeling, which can greatly influence the performance of GP-based methods.
We rerun the main experiments in Section 5.1, where everything remains the same, but the kernel is replaced by the radial basis function kernel (RBF) and the rational quadratic kernel (RQ).
Note that this is applied to both TRLSE and Straddle, which are GP-based methods, not for HLSE.
As seen on both Figures \ref{fig:rbf} and \ref{fig:rq}, TRLSE still remains better than the baselines on most experiments for both RBF and RQ, asserting the robustness of TRLSE across different kernels.

\begin{figure*}[h]
    \centering
    \hspace*{\fill}
    \subfigure[Levy10]{\includegraphics[scale=0.37]{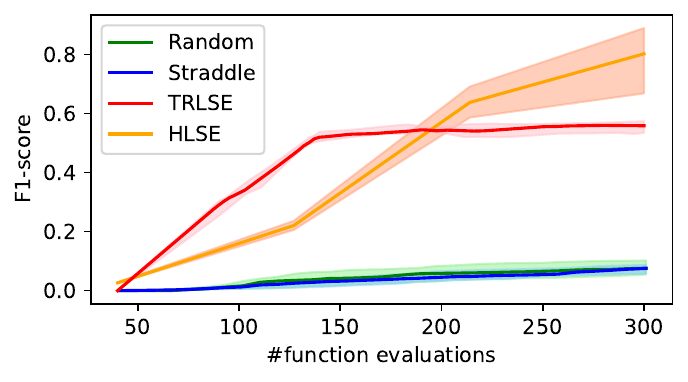}}
    \subfigure[AA33]{\includegraphics[scale=0.37]{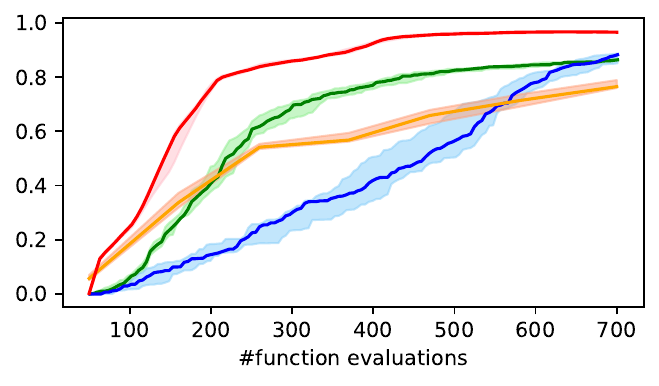}}
    \subfigure[Mazda74]{\includegraphics[scale=0.37]{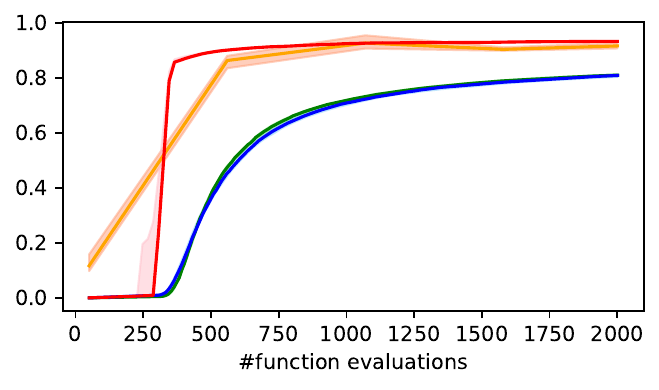}}
    \subfigure[Levy100]{\includegraphics[scale=0.37]{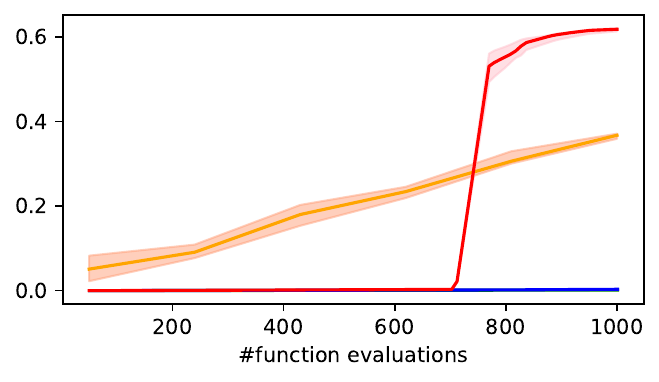}}
    \hspace*{\fill}
    \subfigure[Vehicle124]{\includegraphics[scale=0.37]{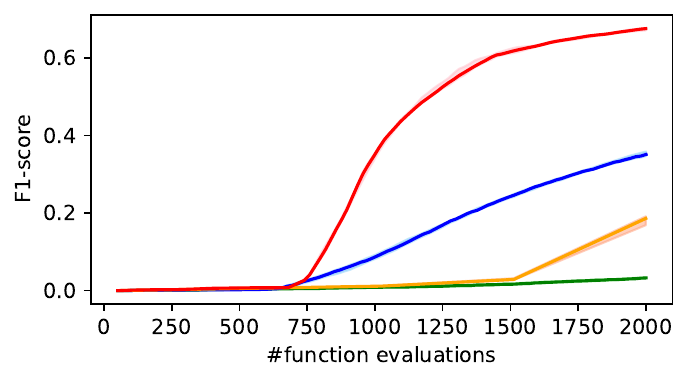}}
    \subfigure[Ackley200]{\includegraphics[scale=0.37]{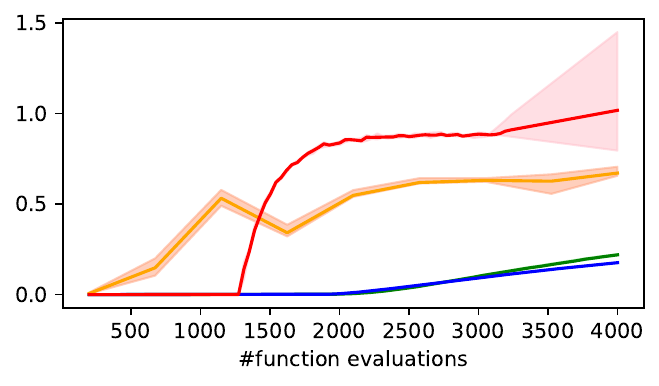}}
    \subfigure[Trid1000]{\includegraphics[scale=0.37]{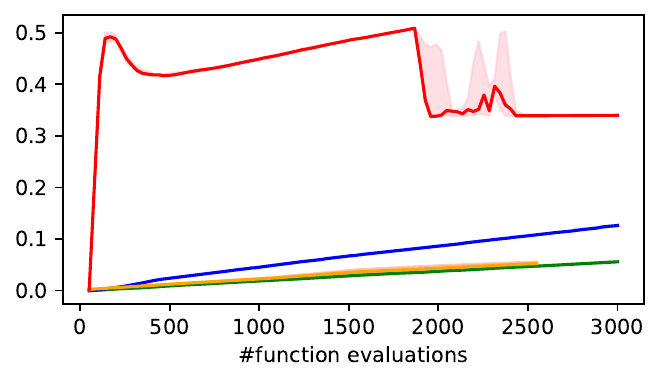}}
    \subfigure[Rosenbrock1000]{\includegraphics[scale=0.37]{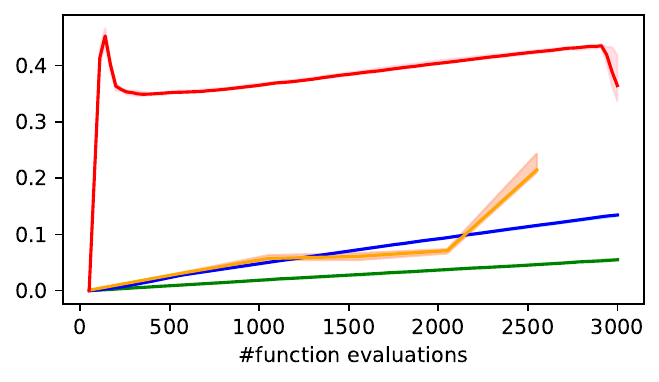}}
    \caption{Using RBF kernel for TRLSE and Straddle}
    \label{fig:rbf} 
\end{figure*}
\begin{figure*}[ht]
    \centering
    \hspace*{\fill}
    \subfigure[Levy10]{\includegraphics[scale=0.37]{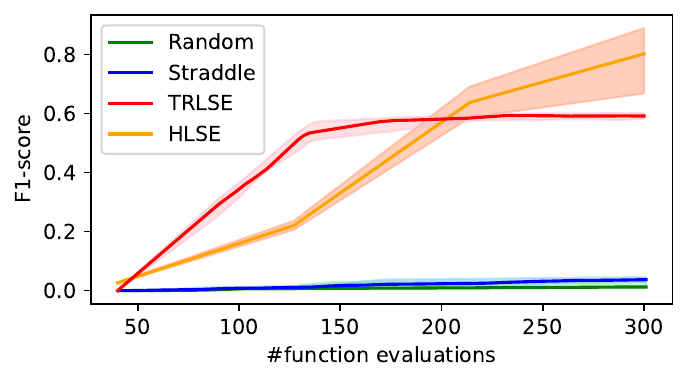}}
    \subfigure[AA33]{\includegraphics[scale=0.37]{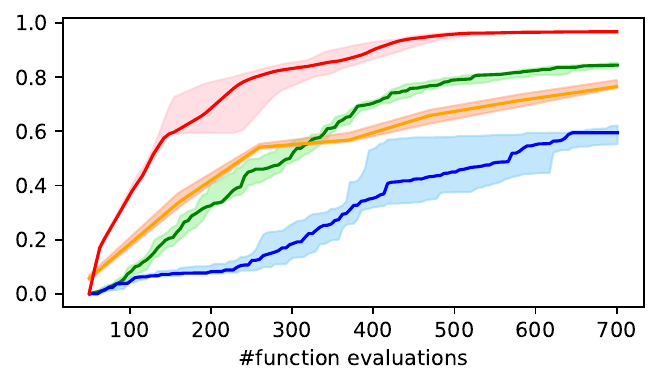}}
    \subfigure[Mazda74]{\includegraphics[scale=0.37]{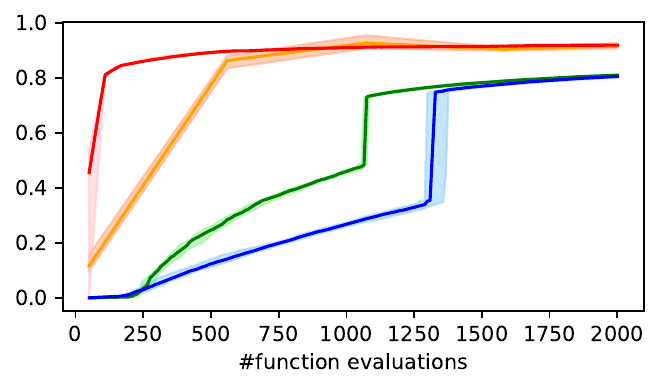}}
    \subfigure[Levy100]{\includegraphics[scale=0.37]{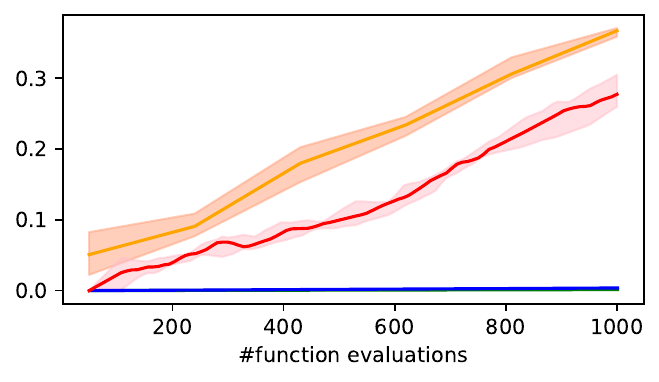}}
    \hspace*{\fill}
    \subfigure[Vehicle124]{\includegraphics[scale=0.37]{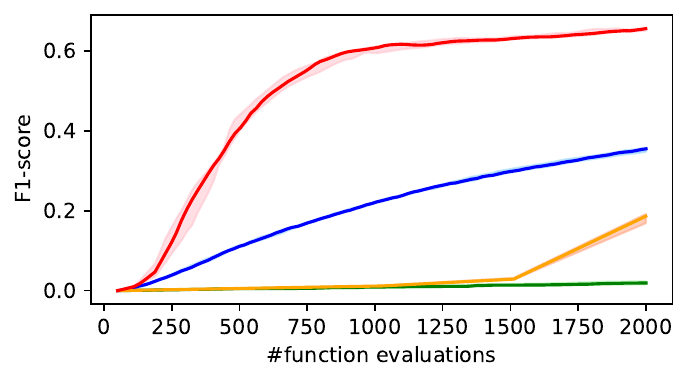}}
    \subfigure[Ackley200]{\includegraphics[scale=0.37]{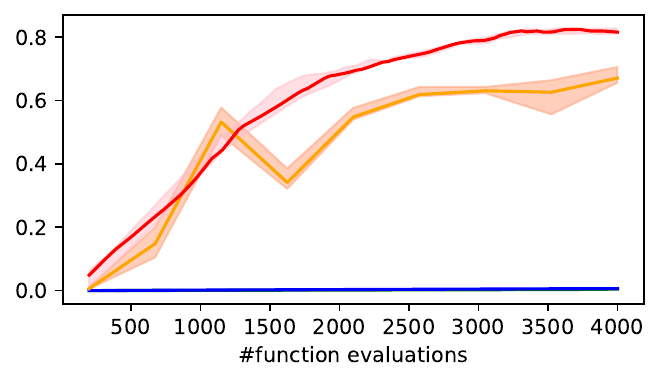}}
    \subfigure[Trid1000]{\includegraphics[scale=0.37]{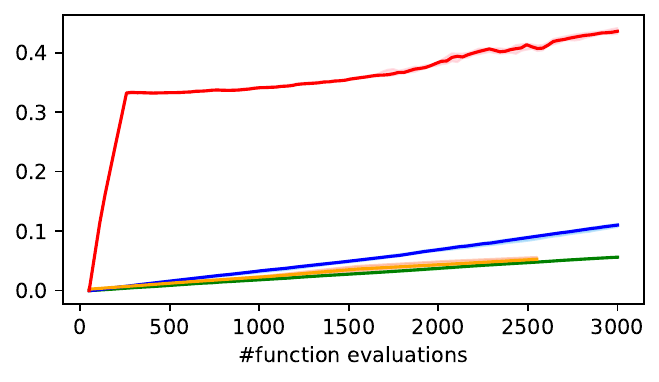}}
    \subfigure[Rosenbrock1000]{\includegraphics[scale=0.37]{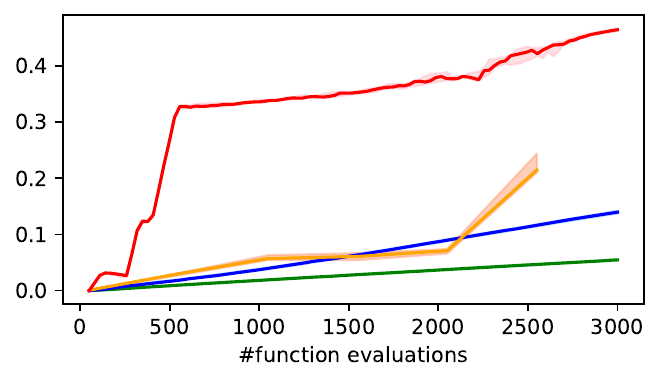}}
    \caption{Using RQ kernel for TRLSE and Straddle}
    \label{fig:rq} 
\end{figure*}
\subsection{Visualizing the Behavior of TRLSE}
\begin{figure*}[ht!]
    \centering
    \subfigure[Iteration 1 - recall=0.21]{\includegraphics[scale=0.20]{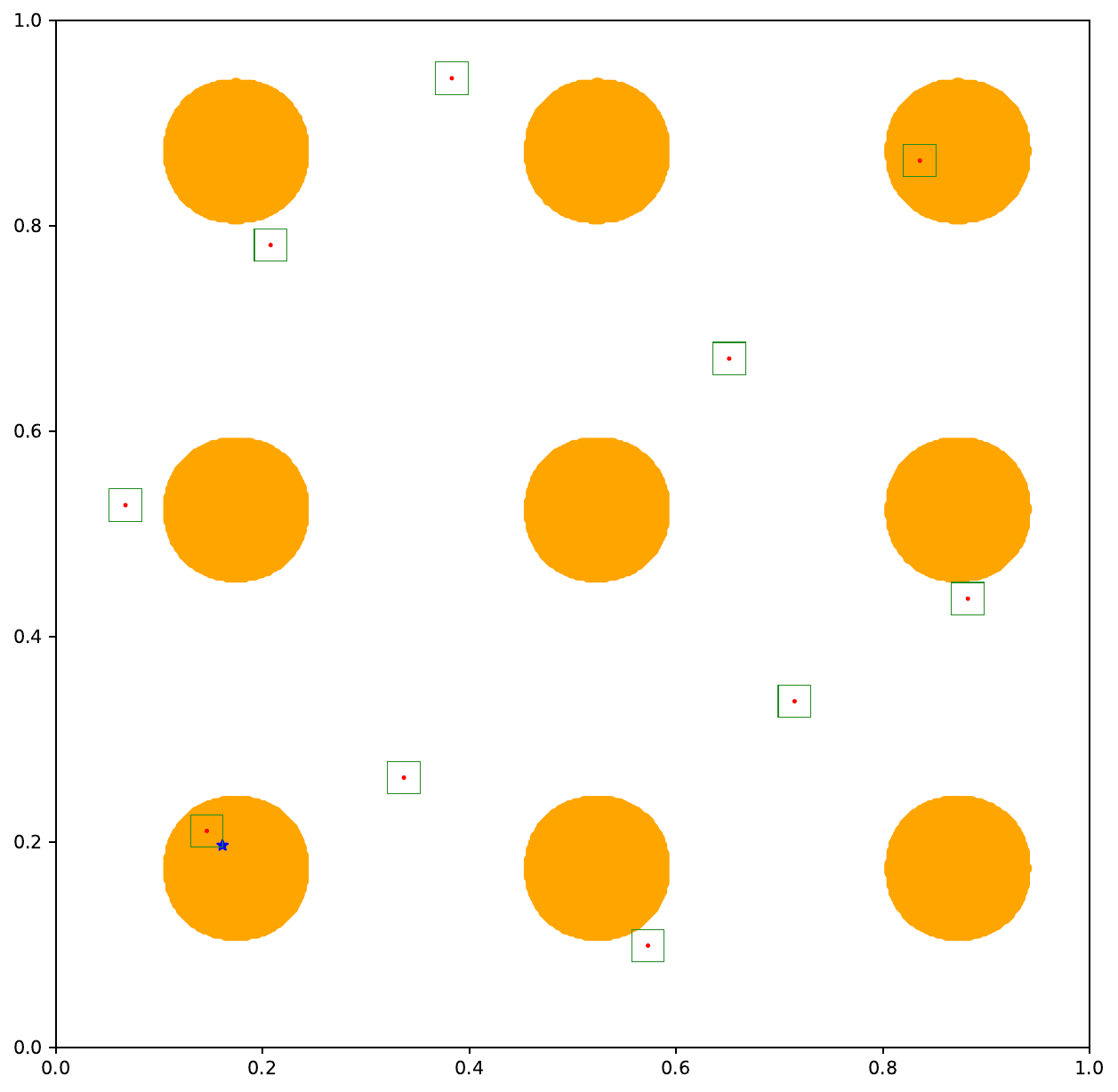}}
    \subfigure[Iteration 2 - recall=0.39]{\includegraphics[scale=0.20]{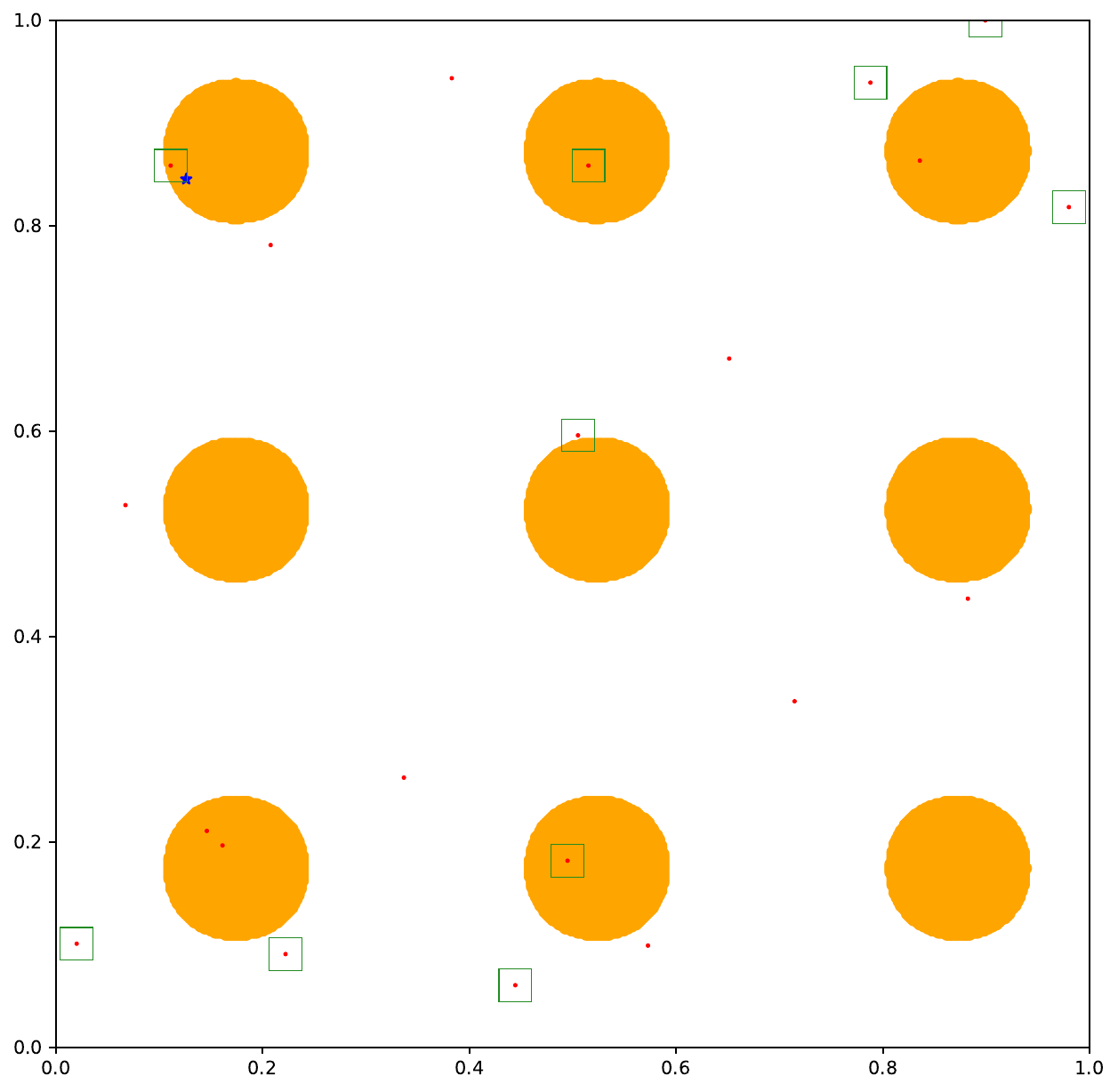}}
    \subfigure[Iteration 3 - recall=0.41]{\includegraphics[scale=0.20]{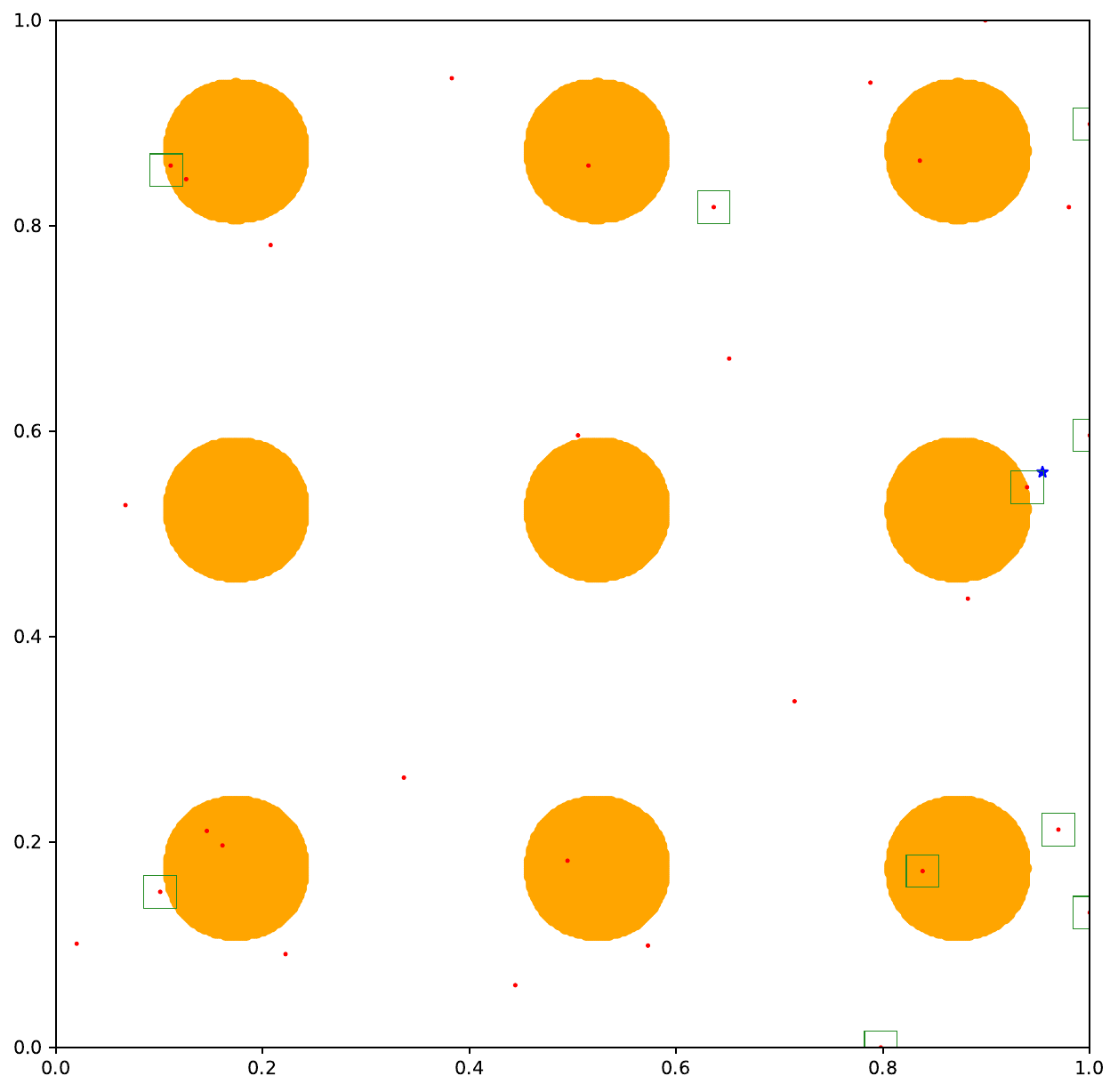}}\\
    \subfigure[Iteration 4 - recall=0.51]{\includegraphics[scale=0.20]{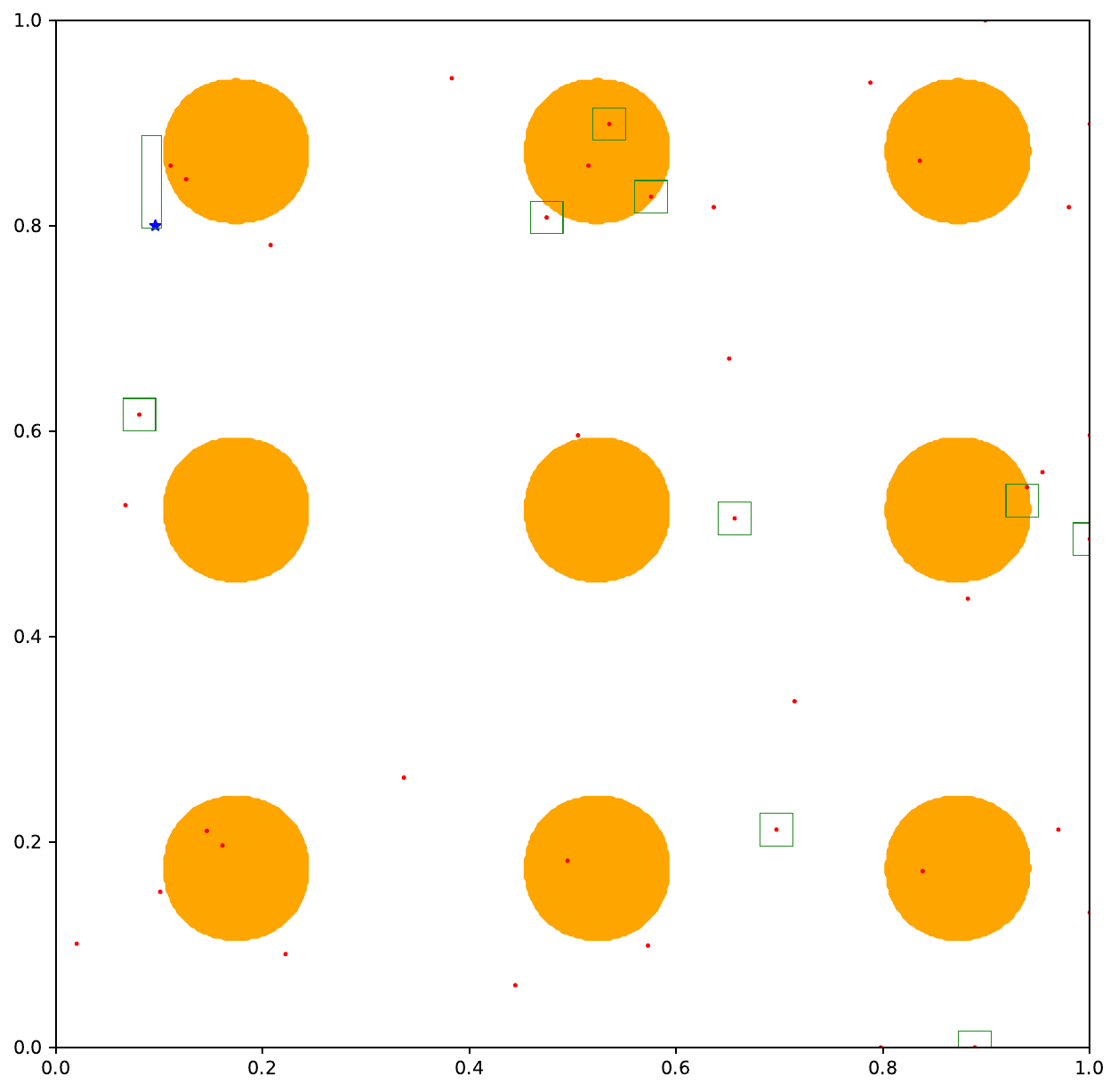}}
    \subfigure[Iteration 5 - recall=0.51]{\includegraphics[scale=0.20]{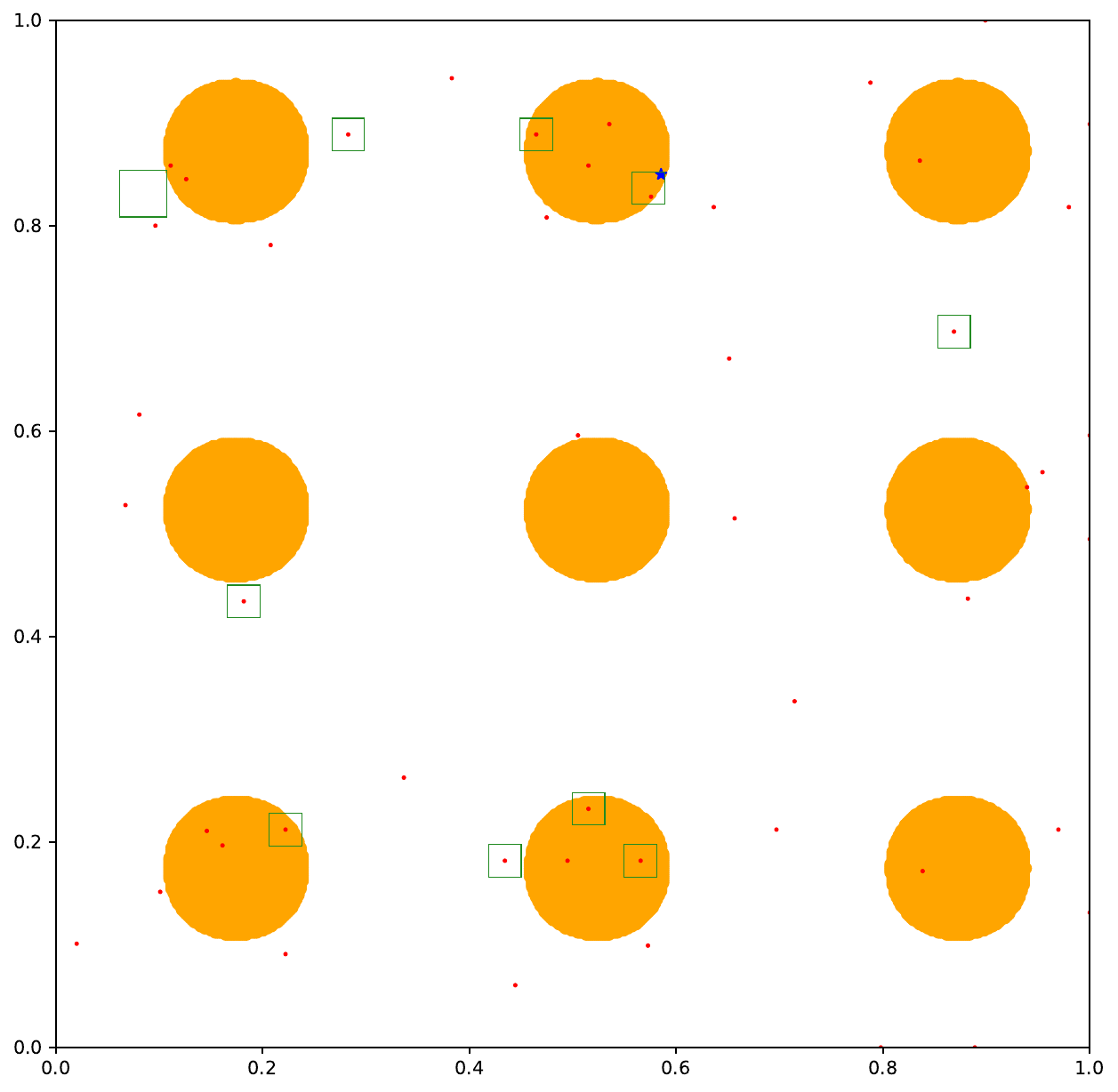}}
    \subfigure[Iteration 6 - recall=0.64]{\includegraphics[scale=0.20]{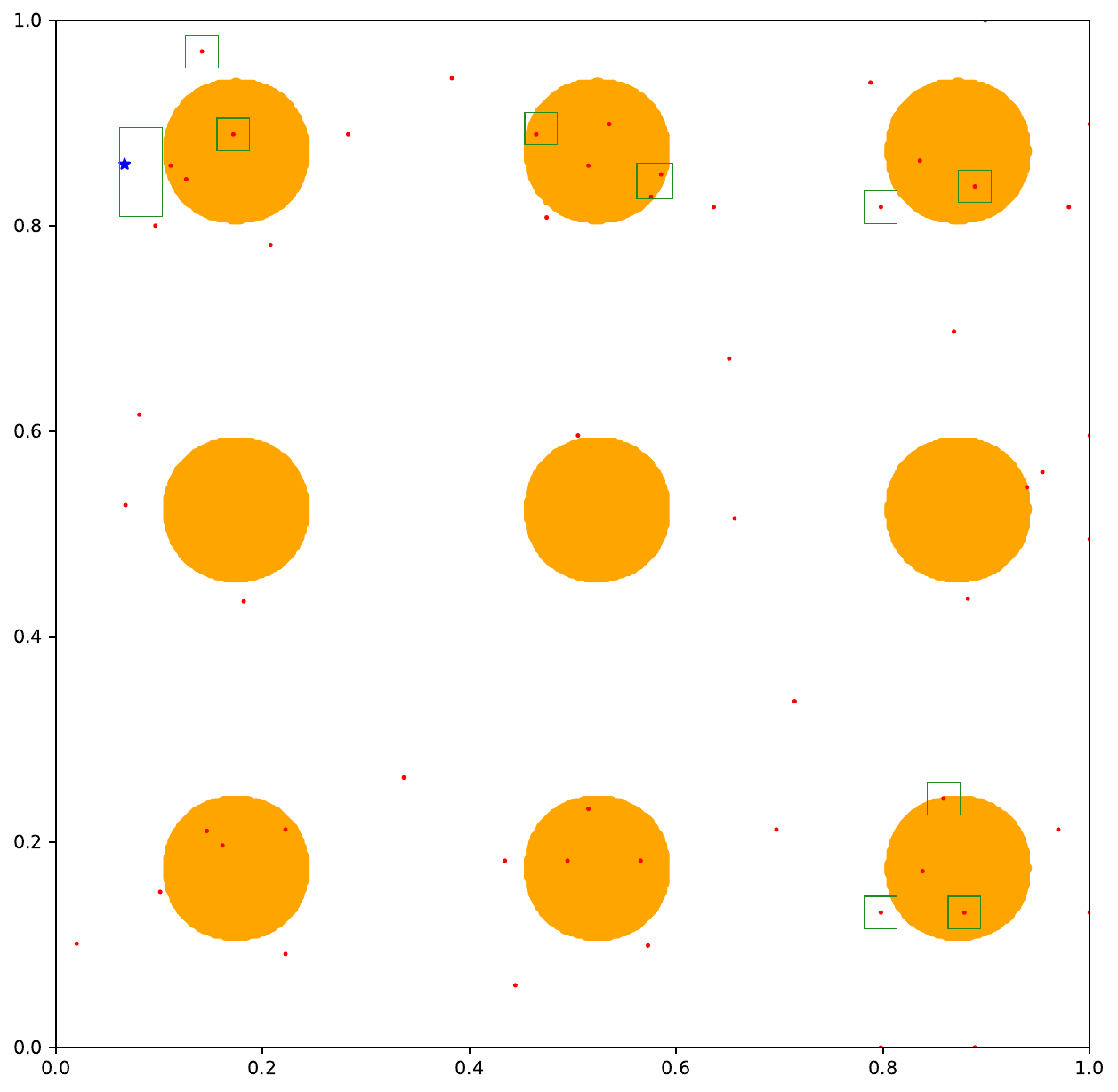}}\\
    \subfigure[Iteration 7 - recall=0.76]{\includegraphics[scale=0.20]{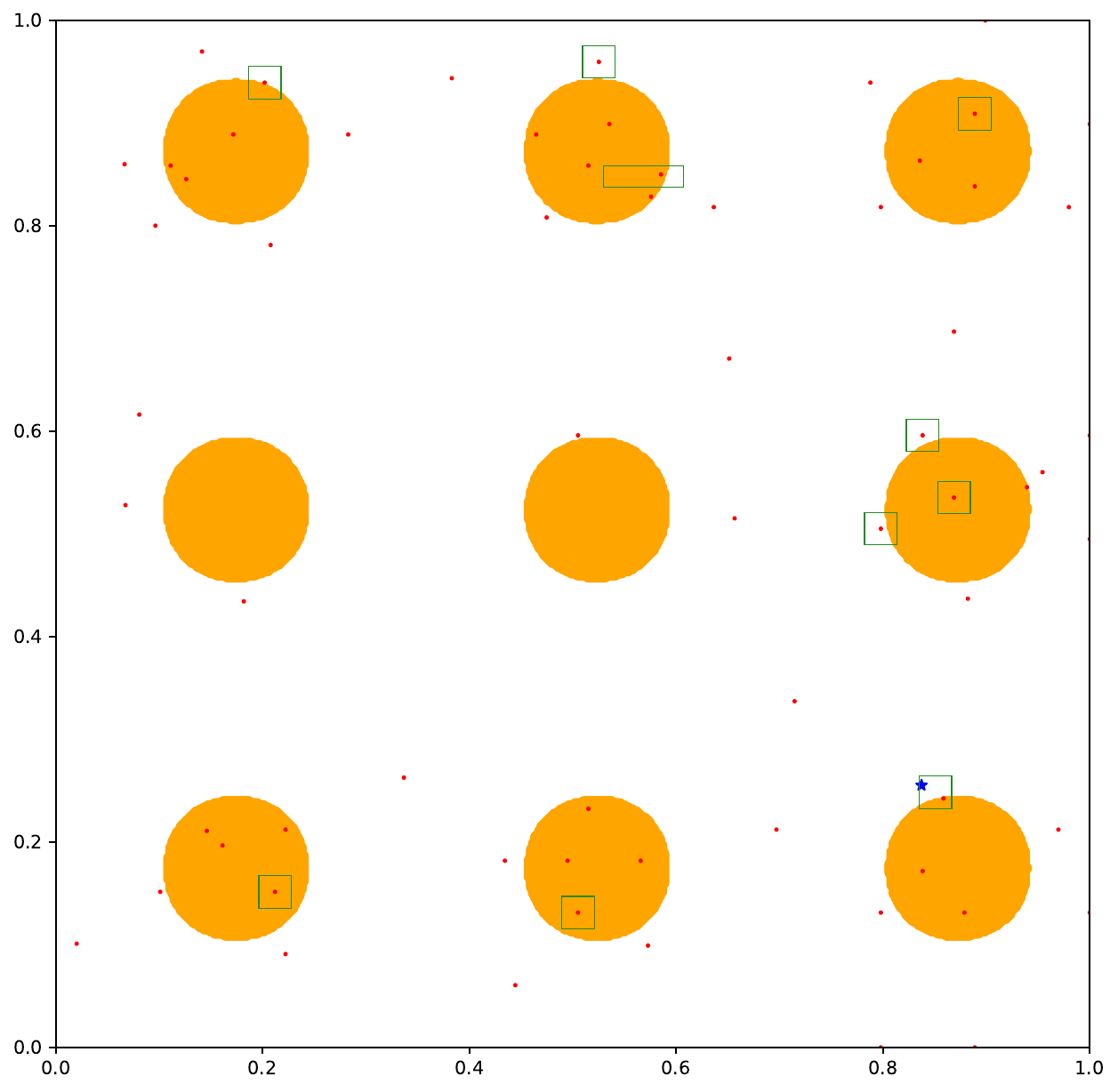}}
    \subfigure[Iteration 8 - recall=0.75]{\includegraphics[scale=0.20]{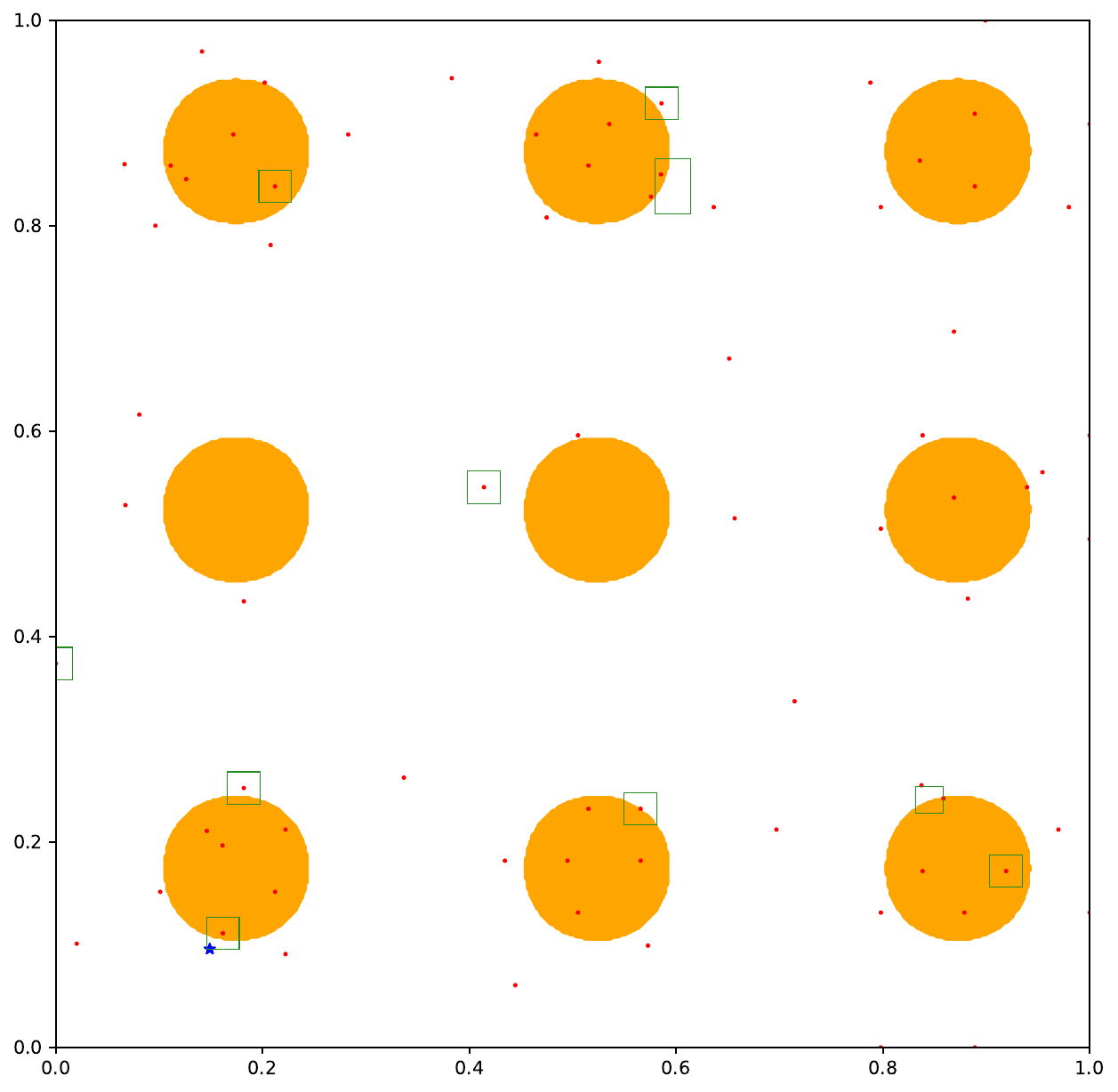}}
    \subfigure[Iteration 9 - recall=0.89]{\includegraphics[scale=0.20]{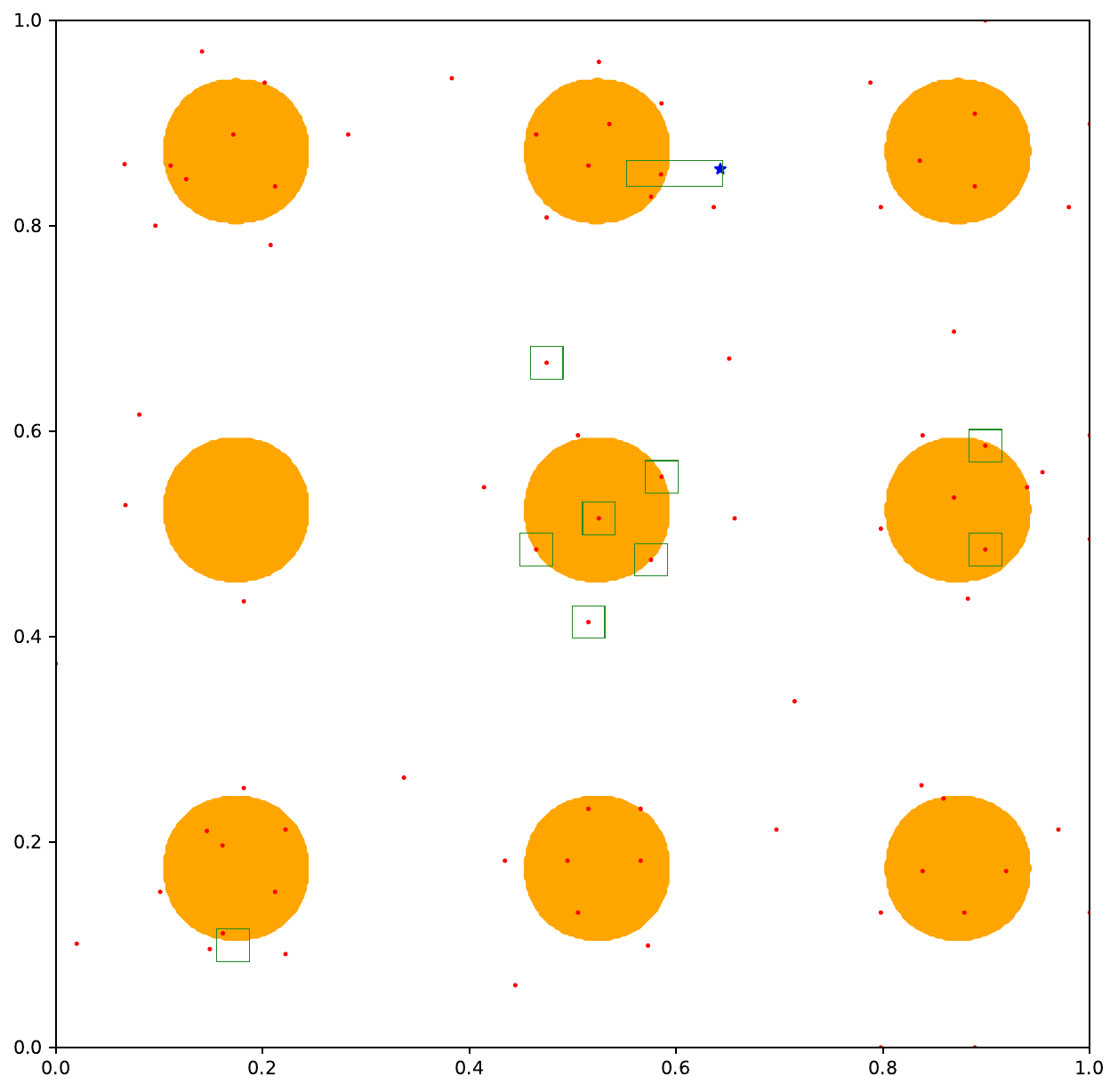}}\\
    \subfigure[Iteration 10 - recall=1.0]{\includegraphics[scale=0.20]{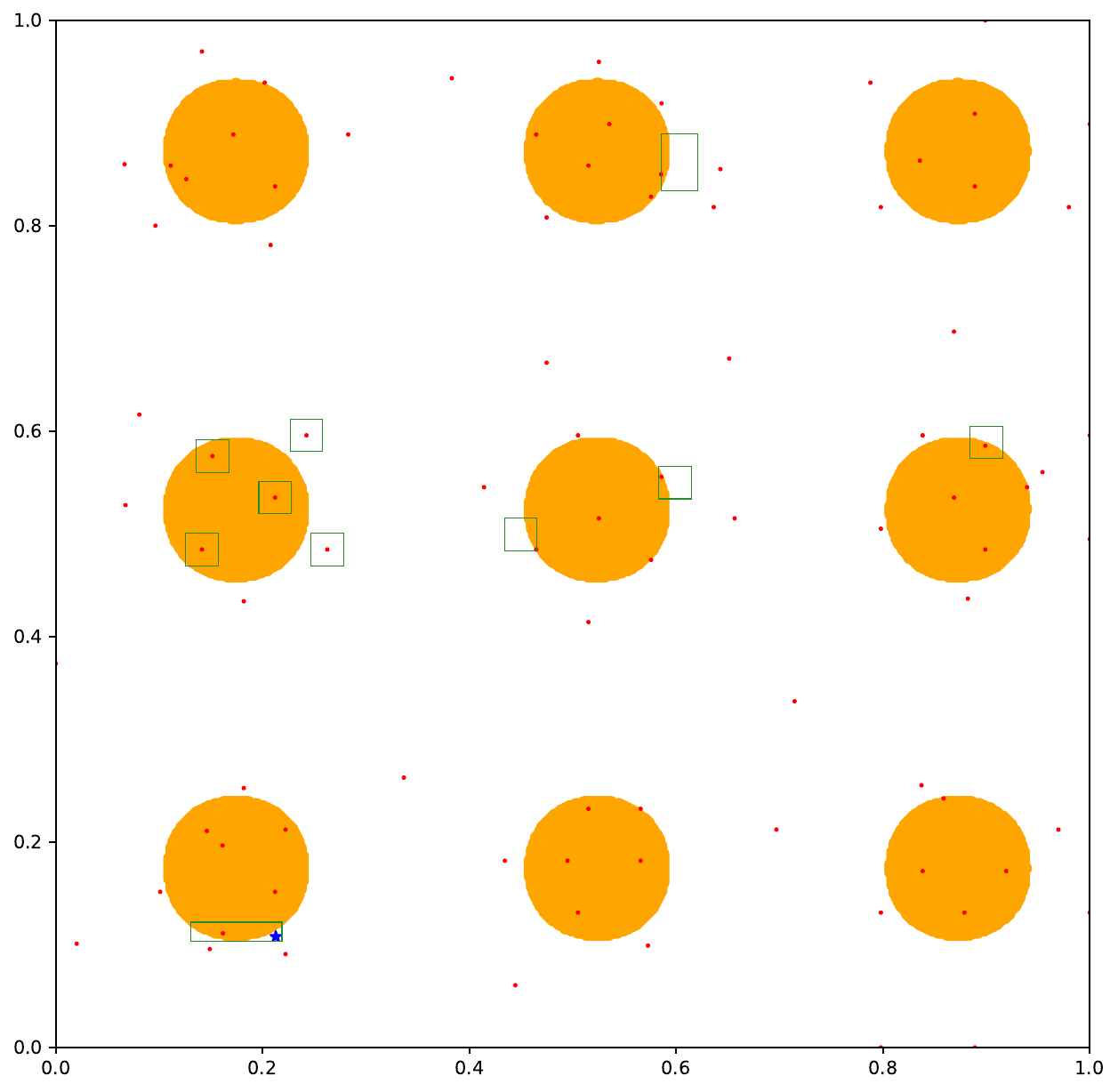}}
    \subfigure[Iteration 11 - recall=1.0]{\includegraphics[scale=0.20]{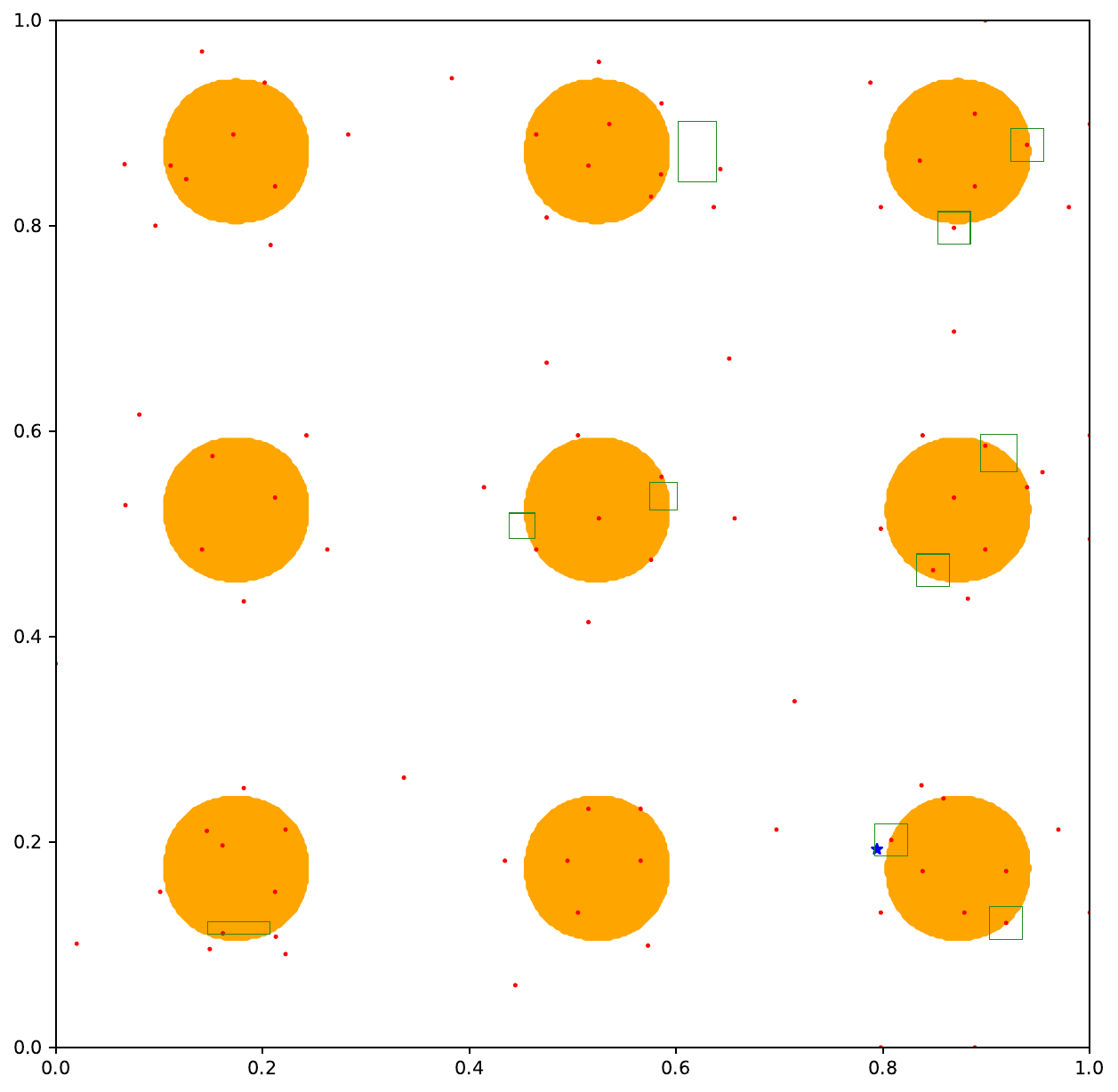}}
    \subfigure[Iteration 12 - recall=1.0]{\includegraphics[scale=0.20]{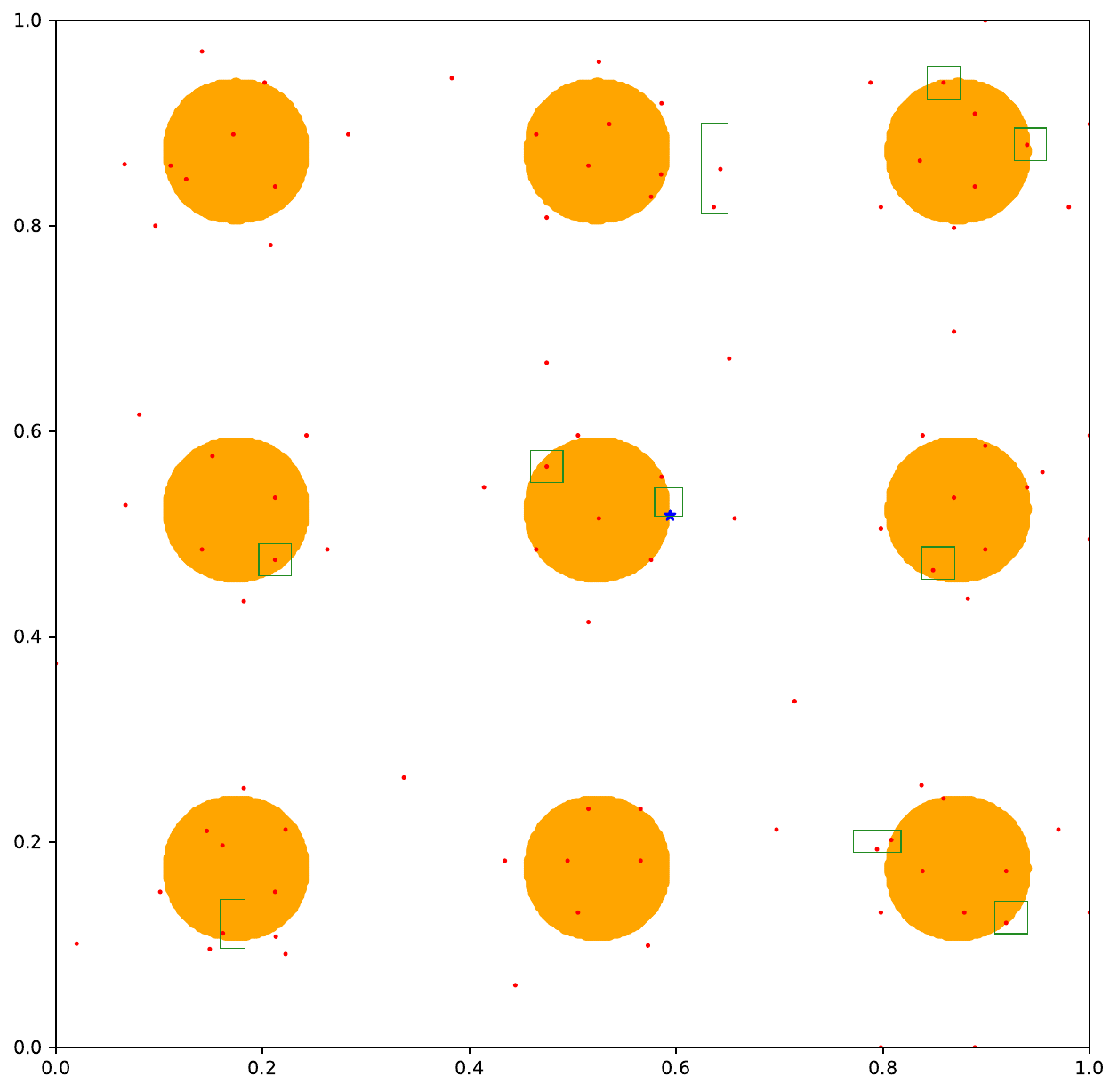}}
    \caption{TRLSE's behavior on MC2D with the corresponding recall for each iteration. 
    The superlevel and sublevel sets are in orange and white, respectively. 
    Red dots represent the training data at the current iteration, and the blue star is the point sampled locally. 
    All TRs are represented by the green boxes.}
    \label{fig:regions} 
\end{figure*}
From Figure \ref{fig:regions}, one can observe how TRLSE behaves on a 2D problem.
A noticeable observation is that TRs far from the threshold boundary quickly get discarded, and as a result, most TRs in later iterations are near the boundary.
Secondly, due to the use of Straddle - a myopic acquisition function, the locally sampled points show explorative behaviors but are still constrained to be inside the TRs, which are relatively close to the boundary.
While these points are not exactly on the boundary, they are near it, which helps with modeling the boundary region, not just locating the boundary, and this is more beneficial for level set classification than simply identifying the boundary's location.
Another observation is that some TRs get discarded prematurely when they have few surrounding data points.
However, this is not essentially bad for the ultimate goal of classifying the whole space.
For example, from iteration 8 to 9, though some TRs, which are close to the boundary, get discarded early, the TRs that replace them discover the boundary region in the center (i.e. five new TRs are spawned in the center in iteration 9).
This explains the sharp jump in recall (from 0.75 to 0.89), showing evidence that this is beneficial for global coverage.